\pdfoutput=1
\documentclass[10pt,twocolumn,letterpaper]{article}

\usepackage[pagenumbers]{iccv} 

\usepackage{multirow}
\usepackage{amsthm}
\usepackage{float}
\usepackage{algorithm}
\usepackage{algpseudocode}
\usepackage{tabularx,booktabs,array}
\newcolumntype{Y}{>{\centering\arraybackslash}X}
\newtheorem{theorem}{Theorem}  

\theoremstyle{definition}               

\theoremstyle{remark}                   

\usepackage{xcolor}
\definecolor{gold}{rgb}{1.0, 0.0, 0.0}
\definecolor{silver}{rgb}{0.8, 0.5, 0.2}
\definecolor{bronze}{rgb}{0.0, 1.0, 0.0}

\usepackage{pifont}
\usepackage[accsupp]{axessibility}

\definecolor{iccvblue}{rgb}{0.21,0.49,0.74}
\usepackage[pagebackref,breaklinks,colorlinks,allcolors=iccvblue]{hyperref}
\usepackage{marvosym}

\def\paperID{13275} 
\def\confName{ICCV}
\def\confYear{2025}

\title{End-to-End Multi-Modal Diffusion Mamba}
\author{%
\makebox[\textwidth][c]{%
  Chunhao~Lu\textsuperscript{1}\quad
  Qiang~Lu\textsuperscript{1\Letter}\quad
  Meichen~Dong\textsuperscript{1,2}\quad
  Jake~Luo\textsuperscript{3}%
}\\[0.3em]
\makebox[\textwidth][c]{%
  \textsuperscript{1}China University of Petroleum-Beijing, \textsuperscript{2}Leyard Optoelectronic, \textsuperscript{3}University of Wisconsin-Milwaukee}\\[0.3em]
\makebox[\textwidth][c]{%
  {\tt\small \{luchunhao,\,meichen.dong\}@student.cup.edu.cn}\quad
  {\tt\small luqiang@cup.edu.cn}\quad
  {\tt\small jakeluo@uwm.edu}%
}%
}

\begin{document}
\maketitle
\begin{abstract}
Current end-to-end multi-modal models utilize different encoders and decoders to process input and output information. This separation hinders the joint representation learning of various modalities. 
To unify multi-modal processing, we propose a novel architecture called MDM (Multi-modal Diffusion Mamba). MDM utilizes a Mamba-based multi-step selection diffusion model to progressively generate and refine modality-specific information through a unified variational autoencoder for both encoding and decoding.
This innovative approach allows MDM to achieve superior performance when processing high-dimensional data, particularly in generating high-resolution images and extended text sequences simultaneously. 
Our evaluations in areas such as image generation, image captioning, visual question answering, text comprehension, and reasoning tasks demonstrate that MDM significantly outperforms existing end-to-end models (MonoFormer, LlamaGen, and Chameleon etc.) and competes effectively with SOTA models like GPT-4V, Gemini Pro, and Mistral.
Our results validate MDM's effectiveness in unifying multi-modal processes while maintaining computational efficiency, establishing a new direction for end-to-end multi-modal architectures.
\end{abstract}    
\vspace{-2em}
\section{Introduction}
\label{sec:intro}
\vspace{-0.5em}
Traditional large-scale multi-modal models~\cite{li2022blip, yu2022coca, alayrac2022flamingo, bao2022vlmo, kim2021vilt, radford2021learning, lu2024explainable, yuan2025autodrive, yuan2025video, qian2025dyncim, ProAPO, MADS, EmDepart, wang2025editor} typically use multiple encoders and decoders to process multi-modal data. This approach makes learning a unified joint representation of the multi-modal data difficult and can significantly slow inference time (as shown in~\cref{fig:1}A). 
To alleviate these problems, end-to-end models without modal-fusion en(de)coder architecture have been proposed (as shown in~\cref{fig:1}B). This approach offers a streamlined, unified processing framework that enhances efficiency and consistency in multi-modal representation learning. Existing end-to-end models follow three primary strategies: (1) Autoregressive models~\cite{team2024chameleon, fuyu-8b, he2024mars, sun2024autoregressive} leverage a single Transformer for both text and image generation, but struggle with the inherent sequential dependency of autoregressive decoding. (2) Hybrid image generation models~\cite{ge2024seed, wu2023next} integrate an additional image synthesis module, improving image quality but introducing extra complexity. (3) Mixed autoregressive-diffusion models~\cite{zhao2024monoformer, zhou2024transfusion, chu2025usp} employ diffusion-based image generation while maintaining an autoregressive framework for text, yet still struggles with unifying multi-modal.
\begin{figure}[t]
  \centering
   \includegraphics[width=1\linewidth]{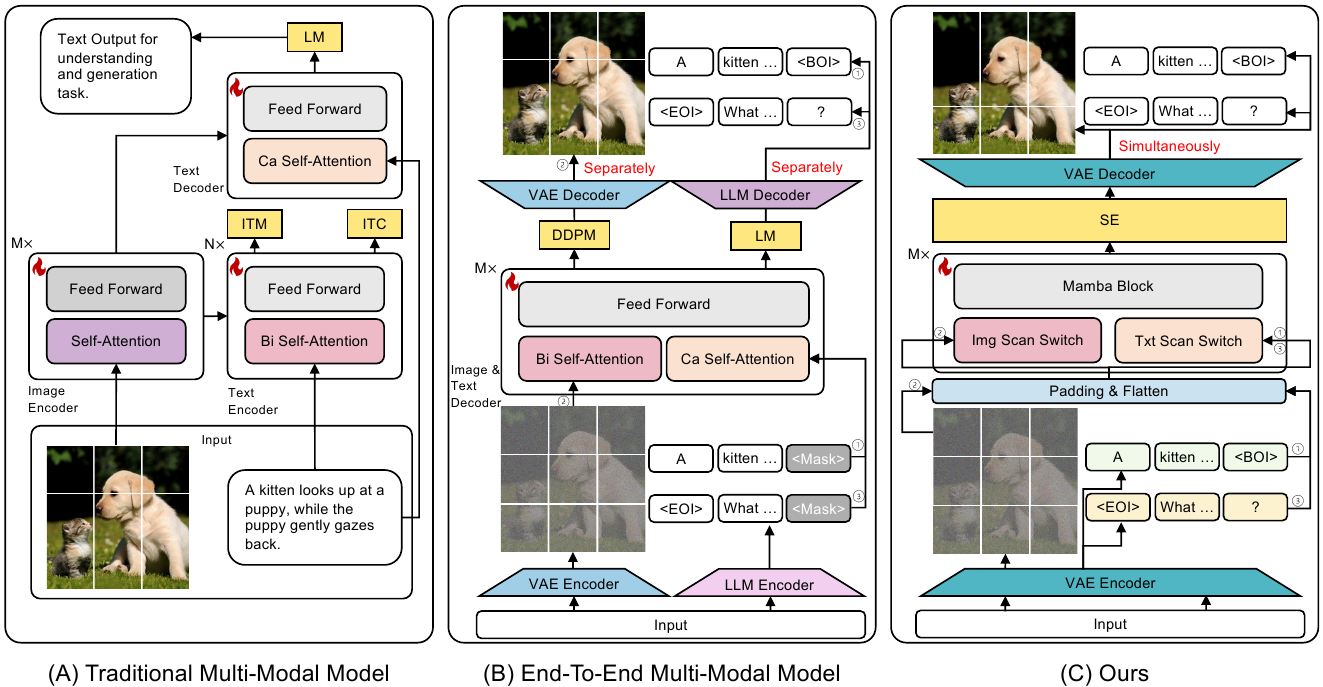}
   \vspace{-2em}
   \caption{Comparison of three types of models.}
   \label{fig:1}
   \vspace{-2em}
\end{figure}

\begin{figure*}[t]
  \centering
   \includegraphics[width=1\linewidth]{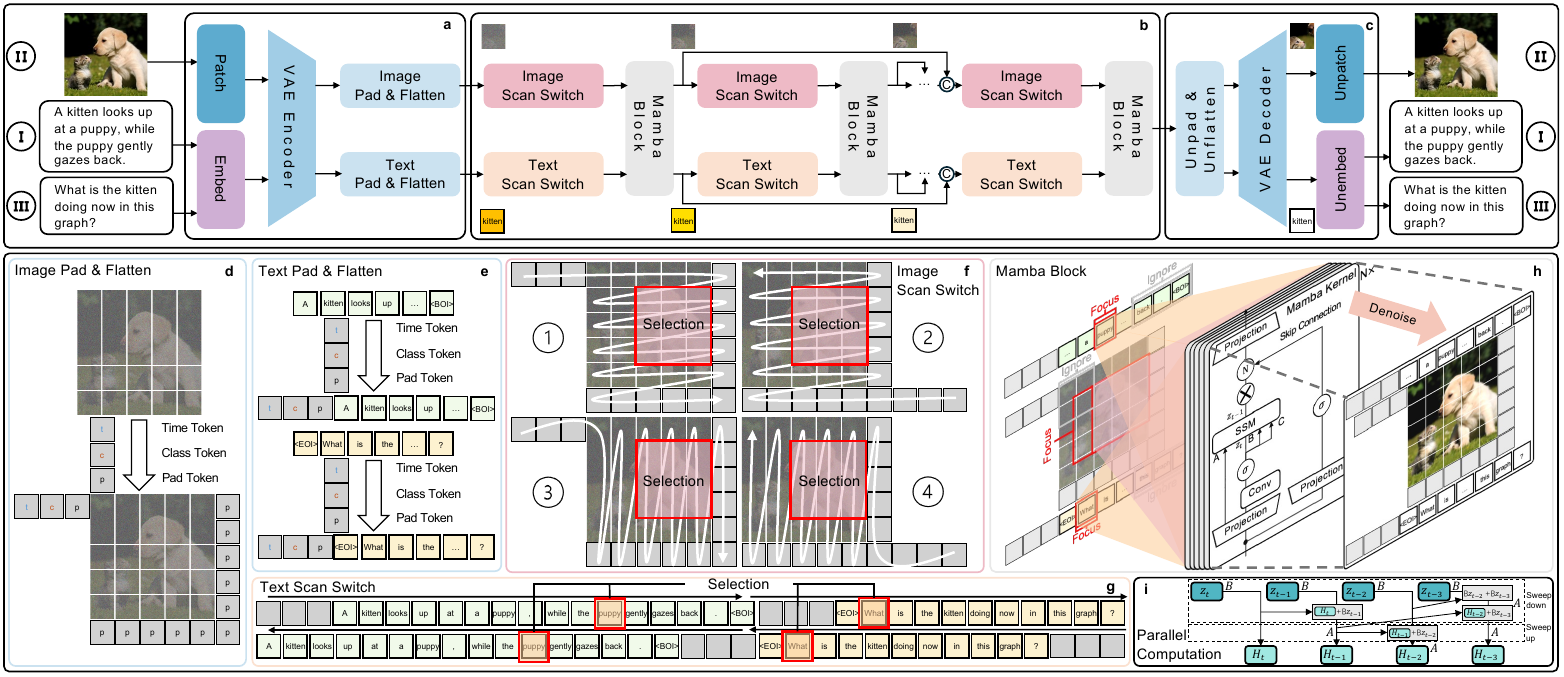}
   \vspace{-2em}
   \caption{Framework of Multi-Modal Diffusion Mamba. MDM first encodes inputs (caption, VQVAE-processed image, question) using VAE (a), while performing padding (class, diffusion timestep, token completion) and flatten operations (d, e). Next, data reconstruction is progressively completed via diffusion mamba operations (b), modeling images and text temporally through scanning processes (f, g) for efficient information selection (red boxes indicate selection). Selected data undergoes computation (i) guided by (h) within the Mamba-2 framework to update model parameters. Finally, the MDM output passes through the VAE decoder (c) to reconstruct real data.}
   \label{fig:2}
   \vspace{-1.5em}
\end{figure*}

Despite recent advancements, Transformer-based end-to-end models face several critical challenges: (1) their quadratic computational complexity makes them inefficient for generating high-resolution image and long-sequence text. Although various studies have attempted to optimize this computational complexity~\cite{wan2023efficient, touvron2023llama, shazeer2019fast, ainslie2023gqa, alberti2023sumformer, gupta2023flurka, pagliardini2023fast, dao2022flashattention, han2023hyperattention, prabhu2024vattention}, the challenge remain substantial. (2) their reliance on multi-objective learning introduces conflicting optimization goals, impeding convergence and hindering effective joint representation learning. In contrast, state-space models like Mamba~\cite{qu2024survey, gu2023mamba} offer a compelling alternative due to their ability to scale linearly with sequence length while effectively capturing long-range dependencies. However, the current multi-modal implementations of Mamba~\cite{liu2024robomamba, yang2024shmamba, qiao2024vl, dong2024fusion, wan2024sigma, yan2024diffusion, teng2024dim, fei2024dimba, hatamizadeh2024mambavision, hu2024zigma} still adopt a multi-objective approach, limiting their capacity for end-to-end joint representation learning.

To effectively process multi-modal data, we propose an end-to-end model called the Multi-Modal Diffusion Mamba (MDM) (as shown in~\cref{fig:1}c). MDM first employs patchify~\cite{dosovitskiy2020image} and embedding to pre-process multi-modal data. Then, it uses a variational autoencoder (VAE)~\cite{kingma2013auto} as a multi-modal encoder, which uniformly maps the multi-modal data to a noisy latent space (as illustrated in~\cref{fig:2}a). MDM constructs a multi-step selection diffusion model based on the Mamba architecture as a uniform decoder for the rapid generation of multi-modal information.

This decoder generates the target text or image step-by-step based on the diffusion process through the multi-step selection diffusion model (as shown in~\cref{fig:2}b). To enhance decoding speed, the decoder employs the Score Entropy Loss~\cite{loudiscrete} as the objective function instead of Markov chain-based~\cite{ho2020denoising} methods for updating the network to handle multi-modal data throughout the diffusion process. The decoder comprises two components: an image and text scan switch, and a Mamba-2 block~\cite{gu2023mamba}. The text scan switch has two modes for sequence modeling (as shown in~\cref{fig:2}f), while the image scan switch has four, based on the settings of DiM~\cite{teng2024dim} (as shown in~\cref{fig:2}e). The scan switches enable the model to capture sequential relationships across various temporal directions in the data. The selection state-space structure in Mamba then analyzes these sequential relationships within the current denoising step. This analysis guides the selection of relevant information to focus on and irrelevant information to ignore, effectively directing the model's denoising process at each step.

Since MDM unifies the modality encoder and decoder, the model is capable of generating an image and text simultaneously. For example, as shown in~\cref{fig:2}h, when generating an image of a dog alongside its description, the scan switch in the decoder first assesses whether the description contains conditions that necessitate image generation. If such conditions exist, the image scan switch is activated. Consequently, the model directs its selection to the image patches corresponding to the dog during each denoising step. This targeted focus guides the model to effectively denoise relevant pixels while disregarding other areas of the image. A similar selection process is employed for text data. Ultimately, the data, once denoised via the t-step diffusion process, is reconstructed into authentic text (or an image) through the VAE decoder simultaneously. The main contributions of this paper are as follows.

1) We introduce the Multi-Modal Diffusion Mamba (MDM), an end-to-end model that achieves a computational complexity of $\mathcal{O}(MLN^2)$, outperforming previous end-to-end models like MonoFormer~\cite{zhao2024monoformer}, which operate at $\mathcal{O}(ML^{2}N/G)$. This advancement enables the efficient generation of long-sequence text and high-resolution images.

2) We propose a novel multi-step selection diffusion model that combines autoregressive and diffusion-based generative paradigms into a unified learning objective. This method effectively integrates both paradigms within a diffusion process, generating multi-modal data simultaneously.

3) Our experimental results demonstrate MDM's superior performance in image generation on the ImageNet~\cite{deng2009imagenet} and COCO datasets~\cite{karpathy2015deep}. Additionally, it excels in various tasks, including image captioning on Flickr30K~\cite{young2014image} and COCO~\cite{karpathy2015deep}, VQA on VQAv2~\cite{goyal2017making}, VizWiz~\cite{gurari2018vizwiz}, and OKVQA~\cite{marino2019ok}, as well as text comprehension and reasoning on seven datasets~\cite{zellers2019hellaswag, mihaylov2018can, sakaguchi2021winogrande, clark2018think, clark2019boolq, bisk2020piqa}. Furthermore, MDM shows strong results in math-related world knowledge tasks on GSM8k~\cite{cobbe2021training}, MATH~\cite{hendrycks2021measuring}, and MMLU~\cite{hendrycks2020measuring}.
\vspace{-2.5em}
\section{Related Works}
\label{sec:related}
\vspace{-0.5em}
\subsection{Traditional large multi-modal model}
\vspace{-0.5em}
Most existing LMMs are built by integrating architectures from multiple modalities. SOTA image and video generation models employ pre-trained text encoders to represent input prompts in latent space, which then condition a diffusion model for generating videos and images~\cite{saharia2022photorealistic, wang2025jasmine, lan2025flux, chen2025finger}. Many researchers have adopted this approach, fusing feature representations from multiple pre-trained encoders to enhance model performance across different modalities~\cite{podell2023sdxl, esser2024scaling}. This pattern is also prevalent in visual language models, where pre-trained language models are typically augmented with linear projection layers from other pre-trained en/decoders for training in the text space. Examples include Flamingo~\cite{alayrac2022flamingo} and LLaVA~\cite{liu2024visual} for visual understanding, GILL~\cite{koh2024generating} for visual generation, and DreamLLM~\cite{dong2023dreamllm} for both understanding and generation.

\vspace{-0.6em}
\subsection{End-to-End multi-modal model}
\vspace{-0.5em}
End-to-end models have emerged recently to facilitate joint representation learning while improving training and inference efficiency. It can be categorized into three main types: \\
1) \textbf{The autoregressive model}~\cite{team2024chameleon, fuyu-8b, he2024mars, sun2024autoregressive} utilizes one Transformer with an autoregressive approach to generate images and text. For instance, the Fuyu model~\cite{fuyu-8b} processes image patches directly as input to achieve visual comprehension. Models like Chameleon~\cite{team2024chameleon}, Mars~\cite{he2024mars}, and LlamaGen~\cite{sun2024autoregressive} convert images into discrete sequence tokens, then concatenate them with text.\\
2) \textbf{The hybrid image generation model}~\cite{ge2024seed, wu2023next} addresses the limitations of autoregressive approaches in image generation. While maintaining an autoregressive structure for text generation, the models enhance image quality by incorporating an image-generation network. For example, Seed-x model~\cite{ge2024seed} focuses on enhancing specific aspects of image generation, while Next-GPT~\cite{wu2023next} aims to expand multi-modal capabilities within an end-to-end framework.\\
3) \textbf{The mixed autoregressive-diffusion model}~\cite{zhao2024monoformer, zhou2024transfusion} combines the strengths of previous approaches. It performs text autoregressive generation and image diffusion restoration simultaneously. Models like MonoFormer~\cite{zhao2024monoformer} and Transfusion~\cite{zhou2024transfusion} achieve this by incorporating causal self-attention~\cite{yang2021causal} for text tokens and bidirectional self-attention~\cite{devlin2018bert} for image patches, enabling high-quality multi-modal understanding and generation.

\vspace{-0.5em}
\subsection{Mamba in multi-modal model}
\vspace{-0.5em}
Mamba has emerged as a powerful alternative to Transformer for multi-modal data alignment~\cite{liu2024robomamba, yang2024shmamba, wang2024text, dong2024fusion, wan2024sigma}.
Recent works showcase Mamba's capabilities across different multi-modal applications. VL-Mamba~\cite{qiao2024vl} combines a pre-trained Mamba model for language understanding with a connector module to align visual patches and language tokens. However, these models lack end-to-end training capabilities and struggle to learn unified joint representations. MDM provides a truly end-to-end architecture, enabling rapid generation of high-quality, long sequences.
\vspace{-0.6em}
\section{Multi-step Selection Diffusion Model}
\label{sec:E2EMDM}
\vspace{-0.5em}
The multi-step selection diffusion model enables rapid generation of multi-modal information through two key processes: diffusion $\&$ denoising and selection. During the diffusion $\&$ denoising, the model employs a unified Score Entropy Loss~\cite{loudiscrete}(SE) to gradually reconstruct target data from noise through a series of denoising steps (as illustrated in~\cref{fig:2}b). The selection process enables the model to capture sequential relationships across different temporal dimensions in the latent space, determining which information should be focused on or ignored during each diffusion denoising step (as shown in~\cref{fig:2}h). 

\vspace{-0.5em}
\subsection{Diffusion \& Denoising}
\vspace{-0.5em}
The diffusion $\&$ denoising process comprises two main components: diffusion and denoising. The diffusion component can be expressed by the following equation:
\begin{equation}
    z_{n,t}^{g} = \sqrt{\bar{\alpha}_{t}^{g}}z_{n,0}^{g} + \sqrt{1-\bar{\alpha}_{t}^{g}}\epsilon_{n,t}^{g},
    \label{eq:1}
\end{equation}
where $g$ denotes either image patch or text embedding, and $z_{n,0}^{g}$ represents the latent space vector of the $n$-th image patch or text embedding, obtained through VAE sampling~\cite{kingma2013auto}. $z_{n,t}^{g}$ is derived from $z_{n,0}^{g}$ after $t$ steps of noise addition; $\epsilon_{n,t}^{g}\sim\mathcal{N}(0,I)$ represents the added noise; $\bar{\alpha}_{t}^{g}=\prod_{k=1}^{t}\alpha_{k}^{g}$, $\alpha_{k}^{g}=1-\beta_{k}^{g}$, and $\{\beta_{k}^{g}\in (0,1)\}_{k=1}^{T}$ are Gaussian distribution hyperparameters controlling the forward diffusion noise. Following the diffusion Markov principle~\cite{ho2020denoising}, $t$-step forward diffusion process can be characterized by conditional probabilities as follows:
\begin{equation}
    p(z_{n,t}^{g}|z_{n,0}^{g}) = \mathcal{N}(z_{n,t}^{g}; \sqrt{\bar{\alpha}_{t}^{g}}z_{n,0}^{g}, (1-\bar{\alpha}_{t}^{g})I),
    \label{eq:2}
\end{equation}
which means that given $z_{n,0}^{g}$, $z_{n,t}^{g}$ follows a Gaussian distribution with $\sqrt{\bar{\alpha}_{t}^{g}}z_{n,0}^{g}$ as mean and $(1-\bar{\alpha}_{t}^{g})I$ as variance.

In the classic diffusion denoising component~\cite{ho2020denoising}, the model needs to learn the posterior $p(z_{n,t-1}^{g}|z_{n,t}^{g})$ to gradually reconstruct the data. Since $p(z_{n,t}^{g}|z_{n,0}^{g})$ follows a Gaussian distribution, we can assume that the approximate distribution of the denoising process is:
\begin{equation}
    p_{\theta}(z_{n,t-1}^{g} | z_{n,t}^{g}) = \mathcal{N}(z_{n,t-1}^{g}; \mu_{\theta}(z_{n,t}^{g}), (\sigma_{\theta,n}^{g})^2).
\label{eq:3}
\end{equation}
where $\mu_{\theta}(z_{n,t}^{g})$ and $\sigma_{\theta,n}^{g}$ represent the model predicted noise mean and variance at the $t$-th denoising step.

This method achieves the gradual recovery of data by optimizing the conditional probability of each time step by maximum likelihood. However, Markov chain-based~\cite{ho2020denoising} methods limit computational efficiency in high-dimensional spaces and are difficult to extend to discrete data.

To further optimize the denoising process, this paper uses SE~\cite{loudiscrete} as the optimization target. It is a generalized score matching objective that aims to directly learn the probability density ratio between discrete states. The SE can not only stabilize the diffusion denoising process but also improve the sampling quality through the global information of data distribution. In general form, for any state pair $(z_{n,t}^{g},z_{n,0}^{g})$, define the model's score ratio $s_{\theta}(z_{n,t}^{g})$, which represents the relative probability of transferring from $z_{n,t}^{g}$ to $z_{n,0}^{g}$. SE is defined as:
\begin{equation}
    \begin{split}
        se = \sum_{y\in z_{n,0:t-1}^{g}} \omega_{z_{n,t}^{g}}^{g} \bigg( & s_{\theta}(z_{n,t}^g) - \frac{p_{data}(y)}{p_{data}(z_{n,t}^{g})} \log{s_{\theta}(z_{n,t}^g)} \\
        & + K\left(\frac{p_{data}(y)}{p_{data}(z_{n,t}^{g})}\right) \bigg),
    \end{split}
\label{eq:4}
\end{equation}
where $\omega_{z_{n,t}^{g}}^{g}$ is the weight of the loss term, which is used to balance the loss of different states. $K(a) = a(\log{a}-1)$ is a normalization term that ensures the loss is non-negative. $\frac{p_{data}(y)}{p_{data}(z_{n,t}^{g})}$ represents the actual score ratio. $p_{data}(y)$ and $p_{data}(z_{n,t}^{g})$ are the actual data distributions of the former noisy state and the current noisy state. The actual score ratio calculation relationship is shown in~\cref{th_1}.

\begin{theorem} \label{th_1}
    According to Bayes' theorem and the Gaussian distribution density formula, the following calculation relationship of $\frac{p_{data}(y)}{p_{data}(z_{n,t}^{g})}$ is obtained:
\begin{equation}
    \frac{p_{data}(y)}{p_{data}(z_{n,t}^{g})}=\exp{\Bigg( \frac{\|z_{n,t}^{g}\|^{2}}{2}-\frac{\|z_{n,t}^{g}-\sqrt{\bar{\alpha}_{t}^{g}}z_{n,0}^{g}\|^{2}}{2(1-\bar{\alpha}_{t}^{g})}\Bigg)}.
\label{eq:5}
\end{equation}
\end{theorem}
The proof is provided in Appendix A.

Based on the SE~\cite{loudiscrete}, the model predicted score ratio indicates how the model adjusts the probability of the current state to tend to the original data distribution during the denoising process. The definition is as follows:
\begin{equation}
    s_{\theta}(z_{n,t}^{g}) = \frac{p_{\theta}(z_{n,0}^{g})}{p_{\theta}(z_{n,t}^{g})},
\label{eq:6}
\end{equation}
where the denominator represents the probability of the current noise state and the numerator represents the original state probability estimated by the model. According to~\cref{th_3}, the model uses $softmax$ for normalization ensuring numerical stability and enabling gradient optimization when predicting the score ratio.

\begin{theorem} \label{th_3}
    Given the denoising process modelled by a score-based probability ratio function $s_{\theta}(z_{n,t}^{g})$, defined as~\cref{eq:6}, this paper defines a learnable approximation using a parameterized score function $f_{\theta}$, such that the probability ratio can be estimated as:
    \begin{equation}
        s_{\theta}(z_{n,t}^{g}) = \frac{\exp{(f_{\theta}(z_{n,t}^{g},z_{n,0}^{g}))}}{\sum_{y\in z_{n,0:t-1}^{g}}\exp{(f_{\theta}(z_{n,t}^{g},y))}},
    \label{eq:7}
    \end{equation}
\end{theorem}
The proof is provided in Appendix A.

\vspace{-0.5em}
\subsection{Selection}
\label{sec:sle}
\vspace{-0.5em}
The selection process comprises two key steps: scan switch and selection. The scan switch mechanism captures temporal relationships between adjacent image patches (or text embeddings) by generating latent space representations with $k$ different sequential relationships, such as four image patch sequences and two text embedding sequences illustrated in~\cref{fig:2}fg. The mechanism creates $k$ temporal sequences $S=\lbrace \langle z_{1,t}^g, z_{2,t}^g, \dots, z_{i,t}^g \rangle \rbrace_k$.  

The selection step then analyzes these different sequential relationships at the current denoising step $t$ to determine which information should be focused on or ignored, thereby guiding the model's denoising direction in each diffusion step. The selection step chooses $j$ items $z_{n,t}^g$ from each sequence in $S$ according to the following \cref{th_2}. So, the selection step obtain $k$ selection sequences with different lengths, i.e., $S'=\lbrace \langle z_{j_1,t}^g, z_{j_2,t}^g, \dots, z_{j,t}^g \rangle \rbrace_k$ and $S'\in S$.  

\begin{theorem} \label{th_2}
    To achieve the optimal \textbf{s}core \textbf{e}ntropy \cite{loudiscrete} which is demonstrated on~\cref{eq:4}, the selection step choose $j$ items where each $z_{n,t}^g$ satisfies $se = 0$, i.e.,
    \begin{equation}
        s_{\theta}(z_{n,t}^g) \approx \frac{p_{data}(y)}{p_{data}(z_{n,t}^{g})}
    \label{eq:20}
    \end{equation}
\end{theorem}
The proof is provided in Appendix A.

\vspace{-0.5em}
\section{Architecture}
\vspace{-0.5em}
The neural network architecture consists of two primary components: a VAE noisy latent encoder~\cite{kingma2013auto} and a multi-step selection diffusion decoder, as illustrated in~\cref{fig:2}ab. The encoder first processes image data $X_{img}$ through patchify~\cite{dosovitskiy2020image} operations and processes text data $X_{txt}$ through tokenization based on SentencePiece with Unigram BPE~\cite{kudo2018sentencepiece} and embedding operations, then uniformly maps them to the latent space before applying forward noise.

The decoder, based on the multi-step selection diffusion model, leverages Mamba to achieve unified learning objectives while enhancing computational efficiency for processing long sequence data. It employs the SE~\cite{loudiscrete} as the unified objective for both image and text modalities during the diffusion process. During selection, the model captures sequential relationships across different temporal dimensions using various scan switches. These relationships are then efficiently processed through the selection state-space structure in the Mamba Block determining which information to focus on or ignore according to \cref{eq:20}, thereby guiding subsequent diffusion denoising steps (as shown in~\cref{fig:2}h). Finally, the reconstructed image patches and text embeddings are transformed back into their original data formats through a VAE noisy latent decoder~\cite{kingma2013auto}.

\vspace{-0.5em}
\subsection{The noisy latent encoder}
\vspace{-0.5em}
The noisy latent encoder first processes input image $X_{img}$ through patchify and processes text $X_{txt}$ through tokenization and embedding operations to obtain the patch sequence $G(X_{img}/X_ {txt})=\langle g_1,g_2,\dots,g_i \rangle$, where $g_n$ represents the $n$-th image patch or text embedding, respectively. The encoder VAE~\cite{kingma2013auto} generates Gaussian distribution parameters (mean $\mu$ and variance $\sigma$) for these patches, with a similar process applied to text embeddings, i.e., $\textit{VAE}(G) = (\mu, \sigma)$. For each image patch or text embedding $g_n$, its noise $z_n$ is a sample $s_n$ from the distribution $\mathcal{N}(\mu,\sigma)$ with the addition noise $\epsilon_n \sim \mathcal{N}(0,1)$, i.e, $z_n= s_n + \epsilon_n$. Finally, the image $X_{img}$ and text $X_{txt}$ are transformed into the noise sequence $\langle z_1,\cdots, z_i \rangle$ through the above process.

Moreover, three types of learnable padding tokens, time, category, and pad, are inserted into these noise sequences, as illustrated in~\cref{fig:2}de. The time token encodes the current diffusion step, the class token is used to learn the data category, and the pad token represents the start or end position for splitting these noise sequences.

\vspace{-0.5em}
\subsection{The multi-step selection diffusion decoder}
\vspace{-0.5em}
The decoder aims at progressively recovering the image $X_{img}$ or text $X_{txt}$ from noise sequences through two main modules: 1) the multi-step selection diffusion Mamba and 2) the VAE noisy latent decoder. 1) The Mamba is used to recover the patch sequence $\langle g_1,\cdots,g_i \rangle$ from the noise sequence $\langle z_1,\cdots, z_i \rangle$.  2) The VAE noisy latent decoder assembles patches and generates the image $\hat{X}_{img}$ or text $\hat{X}_{txt}$.

\vspace{-0.5em}
\subsubsection{Multi-step selection diffusion Mamba}
\label{sec:mssdm}
\vspace{-0.5em}

The module leverages two components, image/text scan switch and Mamba Block, to implement each denoising step in the multi-step selection diffusion model (\cref{sec:E2EMDM}). 

\textbf{The image/text scan switch component} establishes sequences with different directions to capture different temporal relationships between patches. Following Dim~\cite{teng2024dim}, we implement four distinct scan switches for images (as shown in~\cref{fig:2}f) and two for text (as shown in~\cref{fig:2}g).

\textbf{The Mamba block} is used to select patches from these different scan switch sequences and denoise the input noise $z_{n,t}^g$. The block adopts the state space architecture from Mamba-2~\cite{gu2023mamba}. According to \cref{sec:sle}, it is $s_{\theta}$, where $\theta = \lbrace H_{n,t}^{g}, A, B, C, D, \Delta \rbrace$ represent the state space in the block. The block comprises six key components: 1) linear input and output projection layers, 2) convolution kernel layer, 3) nonlinear activation layer, 4) state space model (SSM), 5) skip connection layer, and 6) normalization layer. 

1) The linear input projection layer reduces the dimensionality of the latent space noise vector while simultaneously applying initial state matrices $A$, $B$, $C$ to the linear projection of input data $z_{n,t}^g$. Additionally, the linear output projection layer represents the denoising step, which transforms the selection noise $z_{n,t}^g$ into $z_{n,t-\Delta t}^{g}$ and outputs it to the next Mamba block according to the following equation.
\begin{equation}
    z_{n,t-\Delta t}^{g} = z_{n,t}^{g}-\frac{\Delta t}{2}[f_{\theta}(z_{n,t}^{g},t)+f_{\theta}(z_{n,t-\Delta t}^{g},t-\Delta t)]
    \label{eq:15}
\end{equation}
where the equation adopts the second-order numerical method of DPM-Solver~\cite{lu2022dpm} to improve sampling accuracy. Details are provided in Appendix B.

2) The convolution kernel layer implements parallel scan switches, routing the initial linear projection of the input and the state matrix's linear projection through the SSM, as shown in~\cref{fig:2}i. The sweep down and sweep up~\cite{gu2023mamba} enable parallel computation between~\cref{eq:8,eq:9,eq:10,eq:11}. 

3) The nonlinear layer enhances model generalization.

4) The SSM lets the Mamba block $s_\theta$ approximate the actual score ratio based on \cref{th_2}. To implement the target, SSM updates the state space $\theta$ by the following equations (based on~\cref{th_2} and details in Appendix A).

\begin{equation}
    H_{n,t}^{g} = \bar{A}H_{n,t-1}^{g} + \bar{B}z_{n,t}^{g}
    \label{eq:8}
\end{equation}
\begin{equation}
    z_{n-1,t}^{g} = CH_{n,t}^{g} + Dz_{n,t}^{g}
    \label{eq:9}
\end{equation}
\begin{equation}
    \bar{A} = \exp{(\Delta A)}
    \label{eq:10}
\end{equation}
\begin{equation}
    \bar{B} = (\Delta A)^{-1}\cdot (\exp{(\Delta A)}-I)\cdot \Delta B
    \label{eq:11}
\end{equation}
where $H_{n,t}^{g}$ represents the hidden state representation, $A$ and $B$ control the evolution of hidden states and latent space noise vector inputs, respectively, $C$ governs the hidden state representation of 
the target output and $D$ manages the nonlinear skip connection for latent space noise vector inputs. $\Delta$ denotes the learnable time parameter.

5) The skip connection layer facilitates input feature reuse and mitigates model degradation.

6) The Normalization layer ensures training stability.

According to \cref{eq:20} in \cref{th_2} and~\cref{eq:4}, the goal of training the Mamba block is:
\begin{equation}
    L_{se} = \mathbb{E}_{z_{n,0}^{g}\sim p_{0},z_{n}^{g}\sim p(\cdot|z_{n,0}^{g})} se = 0
    \label{eq:mamba}
\end{equation}

\vspace{-0.8em}
\subsubsection{The noisy latent decoder}
After applying the diffusion-based denoising process, the recovered latent variable $z_{n,0}^{g}$ is passed to the VAE decoder~\cite{kingma2013auto} as illustrated in~\cref{fig:2}c.
For image reconstruction, the decoder applies an $\ell_2$ loss:
\begin{equation}
    L_{rec}^{img} = \mathbb{E}_{z_{n,0}^{g} \sim q_{\phi}(z | X)}\| X_{img} - \hat{X}_{img} \|^2.
\end{equation}
where $q_{\phi}(z | X)$ represents the posterior distribution of the VAE encoder.

For text, the decoder minimizes the cross-entropy loss:
\begin{equation}
    L_{rec}^{txt} = - \mathbb{E}_{z_{n,0}^{g} \sim q_{\phi}(z | X)} \sum_{t} p(X_{txt}^{(t)} | z_{n,0}^{g}) \log p_{\psi}(\hat{X}_{txt}^{(t)} | z_{n,0}^{g}).
\end{equation}
where $p(X_{txt}^{(t)} | z_{n,0}^{g})$ represents the probability distribution of real text data under the condition of latent variable $z_{n,0}^{g}$. And $p_{\psi}(\hat{X}_{txt}^{(t)} | z_{n,0}^{g})$ represents the probability distribution of the text token generated by the VAE decoder under the condition of the latent variable $z_{n,0}^{g}$.

Besides, a KL divergence regularizes the latent space:
\begin{equation}
    L_{KL} = D_{KL} \left( q_{\phi}(z | X) \| p(z) \right).
\end{equation}
where $p(z)$ represents the prior distribution of the latent variable by VAE, which is assumed to be a standard Gaussian distribution $\mathcal{N}(0,I)$ to regularize the latent variable space and enable it to have smooth generation capabilities.

The final optimization objective integrates VAE reconstruction, KL divergence and SE:
\begin{equation}
    L_{total} = L_{rec}^{img} + L_{rec}^{txt} + \beta L_{KL} + \lambda L_{se}.
\end{equation}
\vspace{-2.5em}
\section{Experiments}
\label{sec:exp}
\vspace{-0.3em}
\subsection{Experimental Setup}
\vspace{-0.2em}
\begin{figure}[t]
  \centering
   \includegraphics[width=1\linewidth]{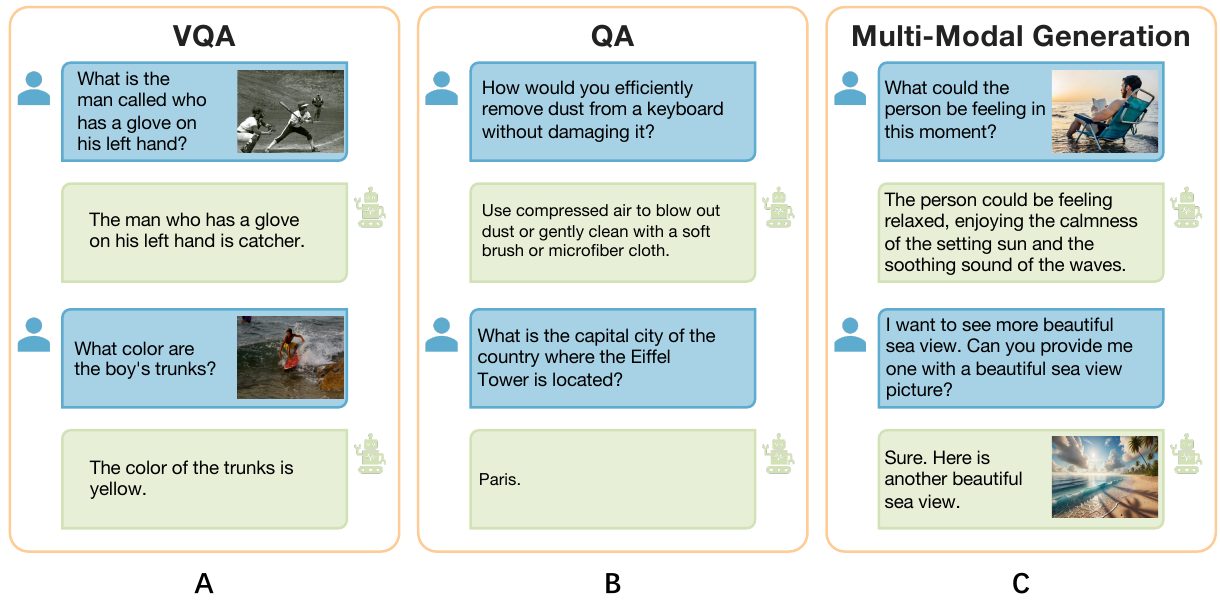}
   \caption{VQA, QA and Multi-Modal generation test from MDM. The results of VQA are part of VQAv2~\cite{goyal2017making}. The QA results are part of PIQA~\cite{bisk2020piqa} and MMLU~\cite{hendrycks2020measuring}. The Multi-Modal generation results are tested with ground-truth data.}
   \vspace{-1em}
   \label{fig:3}
\end{figure}
\begin{figure}[t]
          \centering
           \includegraphics[width=1\linewidth]{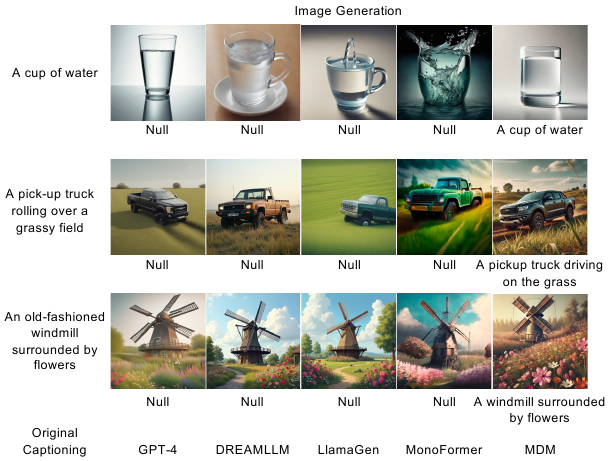}
           \caption{Comparison between each model on generating captioning and image results on COCO dataset. Unlike other models, MDM generates both image and caption data simultaneously.}
           \vspace{-1.5em}
           \label{fig:4}
\end{figure}
\textbf{Model configuration.} 
Our model applies a VAE~\cite{kingma2013auto} as the noisy latent encoder and decoder. Moreover, it integrates the DiM selection state space ~\cite{teng2024dim} in each Mamba block as the diffusion decoder. The resulting model contains 7 billion parameters, with 49 Mamba blocks in the multi-step selection diffusion decoder, each having a dimension of 2048 (Details of parameter settings listed in Appendix C).
\begin{table*}[t]
    \centering
    \fontsize{7}{9}\selectfont
    \renewcommand{\arraystretch}{1.1}  
    \begin{tabularx}{\textwidth}{lcc|*{4}{Y}|*{2}{Y}}
        \toprule
        \multirow{2}{*}{\textbf{Model}} & \multirow{2}{*}{\textbf{Arc}} & \multirow{2}{*}{\textbf{Params}} 
        & \multicolumn{4}{c|}{\textbf{Image Generation with CFG}} 
        & \multicolumn{2}{c}{\textbf{Text-to-Image Generation}} \\
        \cmidrule(lr){4-7} \cmidrule(lr){8-9}
        & & & FID ↓ & IS ↑ & Pre ↑ & Re ↑ & FID ↓ & Gen Eval ↑ \\
        \midrule
        Imagen~\cite{saharia2022photorealistic} & Diff & 7.3B & - & - & - & - & 7.27 & - \\
        ADM~\cite{dhariwal2021diffusion} & Diff & 554M & 10.94 & 101.0 & 0.69 & \textcolor{silver}{0.63} & - & - \\
        CDM~\cite{ho2022cascaded} & Diff & - & 4.88 & 158.7 & - & - & - & - \\
        LDM~\cite{rombach2022high} & Diff & 400M & 3.60 & 147.6 & \textcolor{gold}{0.87} & \textcolor{gold}{0.68} & - & 0.43 \\
        DiT-XL/2~\cite{peebles2023scalable} & Diff & 675M & \textcolor{gold}{2.27} & \textcolor{bronze}{278.2} & 0.83 & 0.57 & - & - \\
        SDXL~\cite{podell2023sdxl} & Diff & 3.4B & - & - & - & - & \textcolor{silver}{4.40} & 0.55 \\
        SD-3~\cite{esser2024scaling} & Diff & 12.7B & - & - & - & - & - & \textcolor{silver}{0.68} \\
        \midrule
        VQGAN~\cite{esser2021taming} & AR & 227M & 18.65 & 80.4 & 0.78 & 0.26 & - & - \\
        ViT-VQGAN~\cite{yu2021vector} & AR & 1.7B & 4.17 & 175.1 & - & - & - & - \\
        \midrule
        NExT-GPT~\cite{wu2023next} & AR & 7B & - & - & - & - & 10.07 & - \\
        Chameleon~\cite{team2024chameleon} & AR & 7B & - & - & - & - & 26.74 & 0.39 \\
        LlamaGen~\cite{sun2024autoregressive} & AR & 3.1B & 2.81 & \textcolor{gold}{311.5} & \textcolor{bronze}{0.84} & 0.54 & \textcolor{gold}{4.19} & - \\
        Transfusion~\cite{zhou2024transfusion} & AR+Diff & 7.3B & - & - & - & - & 6.78 & 0.63 \\
        MonoFormer~\cite{zhao2024monoformer} & AR+Diff & 1.1B & \textcolor{bronze}{2.57} & 272.6 & \textcolor{bronze}{0.84} & 0.56 & - & - \\
        Dual-DiT~\cite{li2024dual} & Diff & 2B & - & - & - & - & 9.40 &\textcolor{bronze}{0.65} \\
        JanusFlow~\cite{ma2024janusflow} & AR+Diff & 1.3B & - & - & - & - & - & \textcolor{gold}{0.70} \\
        Show-O~\cite{xie2024show} & AR+Diff & 1.3B & - & - & - & - & 9.24 & \textcolor{silver}{0.68} \\
        \midrule
        \textbf{MDM} & Diff & 7B & \textcolor{silver}{2.49} & \textcolor{silver}{281.4} & \textcolor{silver}{0.86} & \textcolor{bronze}{0.59} & \textcolor{bronze}{5.91} & \textcolor{silver}{0.68} \\
        \bottomrule
    \end{tabularx}
    \caption{Performance on ImageNet and COCO 256$\times$256. FID, IS, Pre, and Re stands for Frechet Inception Distance, Inception Score, Precision, and Recall, respectively.}
    \vspace{-1em}
    \label{tab:1}
\end{table*}
\begin{table*}[t]
    \centering
    \fontsize{7}{9}\selectfont
    \renewcommand{\arraystretch}{1.1}  
    \begin{tabularx}{\textwidth}{l|*{2}{Y}|*{3}{Y}|*{7}{Y}|*{3}{Y}}
        \toprule
        \multirow{2}{*}{\textbf{Model}} 
        & \multicolumn{2}{c|}{\textbf{IC}} 
        & \multicolumn{3}{c|}{\textbf{VQA}} 
        & \multicolumn{7}{c|}{\textbf{Text Comprehension and Reasoning}} 
        & \multicolumn{3}{c}{\textbf{Math and World}} \\
        \cmidrule(lr){2-3} \cmidrule(lr){4-6} \cmidrule(lr){7-13} \cmidrule(lr){14-16}
        & Flickr & COCO & VQAv2 & VizWiz & OK 
        & HS & OBQA & WG & ARCE & ARCC & BoolQ & PIQA 
        & GSM8k & MATH & MMLU \\
        \midrule
        Llama-2~\cite{touvron2023llama} (7B) & - & - & - & - & - 
        & 77.2 & \textcolor{gold}{58.6} & \textcolor{gold}{78.5} & \textcolor{bronze}{75.2} & 45.9 & \textcolor{bronze}{77.4} & 78.8 
        & 14.6 & 2.5 & 45.3 \\
        Mistral~\cite{jiang2023mistral} (7B) & - & - & - & - & - 
        & \textcolor{bronze}{81.3} & - & \textcolor{silver}{75.3} & \textcolor{gold}{80.0} & \textcolor{gold}{55.5} & \textcolor{gold}{84.7} & \textcolor{silver}{83.0} 
        & \textcolor{bronze}{52.1} & \textcolor{bronze}{13.1} & \textcolor{bronze}{60.1} \\
        Flamingo~\cite{alayrac2022flamingo} (80B) & 75.1 & 113.8 & 67.6 & - & - 
        & - & - & - & - & - & - & - 
        & - & - & - \\
        Gemini Pro~\cite{team2023gemini} & \textcolor{bronze}{82.2} & 99.8 & 71.2 & - & - 
        & \textcolor{silver}{84.7} & - & - & - & - & - & - 
        & \textcolor{silver}{86.5} & \textcolor{silver}{32.6} & \textcolor{silver}{71.8} \\
        GPT4V~\cite{gpt4v} & 55.3 & 78.5 & \textcolor{silver}{77.2} & - & - 
        & \textcolor{gold}{95.3} & - & - & - & - & - & - 
        & \textcolor{gold}{92.0} & \textcolor{gold}{52.9} & \textcolor{gold}{86.4} \\
        InstructBLIP~\cite{liu2024visual} (7B) & \textcolor{silver}{82.4} & 102.2 & - & 33.4 & 33.9 
        & - & - & - & - & - & - & - 
        & - & - & - \\
        mPLUG-Owl~\cite{ye2023mplug} (7B) & 80.3 & 119.3 & - & 39.0 & - 
        & - & - & - & - & - & - & - 
        & - & - & - \\
        TinyLlama~\cite{zhang2024tinyllama} (1.1B) & - & - & - & - & - 
        & 59.2 & 36.0 & 59.1 & 55.3 & 30.1 & 57.8 & 73.3 
        & - & - & - \\
        Pythia~\cite{biderman2023pythia} (12B) & - & - & - & - & - 
        & 52.0 & 33.2 & 57.4 & 54.0 & 28.5 & 63.3 & 70.9 
        & - & - & - \\
        DREAMLLM~\cite{dong2023dreamllm}(7B) & - & 115.4 & 56.6 & \textcolor{bronze}{45.8} & 44.3 
        & - & - & - & - & - & - & - 
        & - & - & - \\
        Emu~\cite{sun2023generative}(7B) & - & 117.7 & 40.0 & 35.4 & 34.7 
        & - & - & - & - & - & - & - 
        & - & - & - \\
        \midrule
        Chameleon~\cite{team2024chameleon}(34B) & 74.7 & \textcolor{bronze}{120.2} & 66.0 & - & - 
        & 74.2 & \textcolor{silver}{51.0} & 70.4 & \textcolor{silver}{76.1} & \textcolor{bronze}{46.5} & \textcolor{silver}{81.4} & 79.6 
        & 41.6 & 11.5 & 52.1 \\
        NExT-GPT~\cite{wu2023next}(7B) & \textcolor{gold}{84.5} & \textcolor{gold}{124.9} & 66.7 & \textcolor{gold}{48.4} & \textcolor{gold}{52.1} 
        & - & - & - & - & - & - & - 
        & - & - & - \\
        Transfusion~\cite{zhou2024transfusion}(7B) & - & 33.7 & - & - & - 
        & - & - & - & - & - & - & - 
        & - & - & - \\
        MonoFormer~\cite{zhao2024monoformer}(1.1B) & - & - & - & - & - 
        & 50.6 & 37.2 & 56.9 & 48.2 & 31.5 & 62.3 & 71.2 
        & - & - & - \\
        Dual-DiT~\cite{li2024dual}(2B) & - & 56.2 & 60.1 & 29.9 & 25.3 & - & - & - & - & - & - & - & - & - & - \\
        JanusFlow~\cite{ma2024janusflow}(1.3B) & - & - & \textcolor{gold}{79.8} & - & - & - & - & - & - & - & - & - & - & - & - \\
        Show-O~\cite{xie2024show}(1.3B) & 67.6 & - & \textcolor{bronze}{74.7} & - & - & - & - & - & - & - & - & - & - & - & - \\
        \midrule
        \textbf{MDM} (7B) & 62.4 & 109.6 & 60.3 & 39.8 & \textcolor{bronze}{47.1} 
        & 70.6 & 41.5 & 68.8 & 55.1 & 46.2 & 65.7 & \textcolor{bronze}{79.9} 
        & 40.5 & 12.1 & 54.4 \\
        \textbf{InstructMDM} (7B) & 75.2 & \textcolor{silver}{122.1} & 66.7 & \textcolor{silver}{46.3} & \textcolor{silver}{51.6} 
        & 74.8 & \textcolor{bronze}{48.3} & \textcolor{bronze}{74.9} & 65.4 & \textcolor{silver}{47.1} & 71.5 & \textcolor{gold}{83.7} 
        & 46.0 & \textcolor{bronze}{13.1} & 59.2 \\
        \bottomrule
    \end{tabularx}
    \caption{Performance on image-to-text and text-to-text tasks. The evaluation of image captioning (IC) and VQA is CIDEr and answer accuracy $\%$ (Flickr is evaluated on 30K and OK represents OKVQA).}
    \vspace{-1.9em}
    \label{tab:2}
\end{table*}

Before the training MDM process, we trained a tokenization model based on SentencePiece (Unigram BPE)~\cite{kudo2018sentencepiece}. The tokenization model can help the model construct a stable text latent variable representation, thereby optimizing the forward diffusion and reverse denoising process. See Appendix D for detailed experimental settings. 

In the training process, we import the DDPM scheduler~\cite{ho2020denoising} and DPM-Solver~\cite{lu2022dpm} to improve the sampling efficiency in the diffusion model. We then use the AdamW optimizer without weight decay, maintaining a constant learning rate of 0.0001. Meanwhile, we keep an EMA of the model weights with a coefficient of 0.9999.

\textbf{Baseline and dataset.}
Our evaluation encompasses four tasks: image generation with classifier‑free guidance~\cite{ho2022classifier} (CFG), text-to-image, image-to-text, and text-to-text generation. For the baseline model training, we train MDM on ImageNet~\cite{deng2009imagenet}, JourneyDB~\cite{sun2024journeydb} and UltraChat~\cite{ding2023enhancing}.
\begin{figure*}[t]
      \centering
       \includegraphics[width=0.8\linewidth]{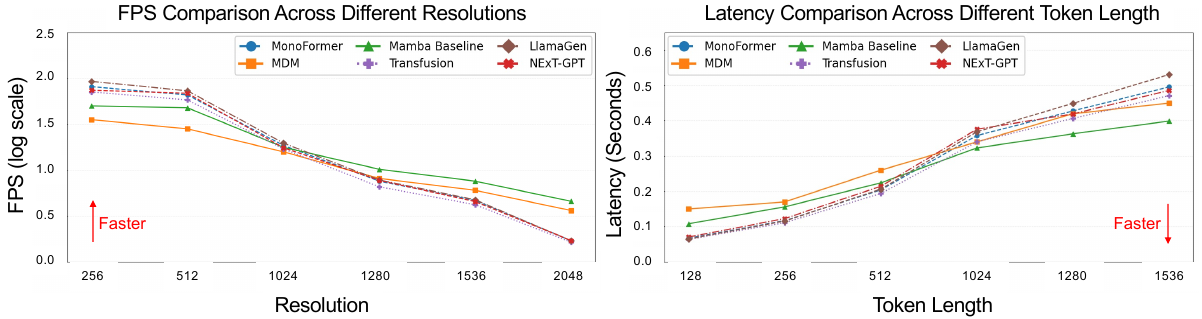}
       \caption{Comparison between Mamba Baseline, MonoFormer, and ours MDM on inference speed test. The left shows the inference speed of the model FPS at different resolutions. The right shows the inference speed of the model latency at different token lengths.}
       \vspace{-1em}
       \label{fig:5}
\end{figure*}

For the image generation and the text-to-image task at $256\times 256$ resolution, we compare the MDM baseline model against established baselines across three categories: diffusion models (Imagen~\cite{saharia2022photorealistic}, ADM~\cite{dhariwal2021diffusion}, CDM~\cite{ho2022cascaded}, LDM~\cite{rombach2022high}, DiT-XL/2~\cite{peebles2023scalable}, SDXL~\cite{podell2023sdxl}, and SD-3~\cite{esser2024scaling}), autoregressive models (VQGAN~\cite{esser2021taming} and ViT-VQGAN~\cite{yu2021vector}), and end-to-end multi-modal models (NExT-GPT~\cite{wu2023next}, Chameleon~\cite{team2024chameleon}, LlamaGen~\cite{sun2024autoregressive}, Transfusion~\cite{zhou2024transfusion}, MonoFormer~\cite{zhao2024monoformer}, Dual-DiT~\cite{li2024dual}, JanusFlow~\cite{ma2024janusflow} and Show-O~\cite{xie2024show}). For the image generation task, we evaluate performance on ImageNet~\cite{deng2009imagenet} using four metrics: Frechet Inception Distance (FID), Inception Score (IS), and Precision/Recall. For the text-to-image task, we evaluate performance on COCO~\cite{karpathy2015deep} using FID and Gen Eval~\cite{ghosh2023geneval}.

For the image-to-text task (image captioning and vision question answering, VQA) and text-to-text task, we employ MDM baseline model and MDM instruction model by visual instruction tuning~\cite{liu2024visual} on multiple datasets: COCO~\cite{karpathy2015deep}, GQA~\cite{hudson2019gqa}, OCR-VQA~\cite{mishraICDAR19}, TextVQA~\cite{singh2019towards}, and VisualGenome~\cite{krishna2017visual}. We evaluate the model against two groups of baselines: traditional models 
and end-to-end multi-modal models.
Performance evaluation of image captioning is conducted on Flickr 30K~\cite{young2014image} and COCO~\cite{karpathy2015deep} datasets using the Consensus-based Image Description Evaluation (CIDEr) metric. And performance evaluation of VQA is conducted on VQAv2~\cite{goyal2017making}, VizWiz~\cite{gurari2018vizwiz}, and OKVQA~\cite{marino2019ok} using answer accuracy rate as the evaluation metric.

For the text-to-text task, we evaluate the model on text comprehension and reasoning tasks using HellaSwag~\cite{zellers2019hellaswag}, OpenBookQA~\cite{mihaylov2018can}, Wino-Grande~\cite{sakaguchi2021winogrande}, ARCEasy, ARCChallenge~\cite{clark2018think}, BoolQ~\cite{clark2019boolq}, and PIQA~\cite{bisk2020piqa}. We also evaluate the model on math and world knowledge tasks using GSM8K~\cite{cobbe2021training}, MATH~\cite{hendrycks2021measuring}, and MMLU~\cite{hendrycks2020measuring}. The evaluation metrics for all the tasks are accuracy rates.

\vspace{-0.1em}
\subsection{Experimental Results}
\textbf{Image Generation.} In the image generation task on ImageNet, MDM achieves top-three rankings across all evaluation metrics: second in FID, IS, and Precision, and third in Recall when compared against one-modal diffusion models and end-to-end multi-modal models (see \cref{tab:1}). MDM demonstrates superior overall performance, notably surpassing other end-to-end multi-modal models in three of the four metrics. In the text-to-image task, we tested the model on the COCO dataset to generate both image and caption data. For the image generation results, we evaluated the FID and Gen Eval performance indicators of the model-generated images. MDM still achieved the top three performance levels and achieved SOTA on Gen Eval.
\begin{table}[t]
    \centering
        \renewcommand{\arraystretch}{1}  
        \fontsize{7}{9}\selectfont
          \begin{tabularx}{\columnwidth}{l|Y|Y|Y}
            \toprule
            \textbf{Model} 
              & \textbf{Image/Text Scan Switch} 
              & \textbf{FPS w log scale$\uparrow$} 
              & \textbf{FID$\downarrow$} \\
            \midrule
            Model w Mamba & \ding{172}\ding{173}\ding{174}\ding{175}/\ding{172}\ding{173} & 1.357 & \textbf{2.49} \\
            Model w Mamba & \ding{172}\ding{173}/\ding{172} & 1.405 & 3.96 \\
            Model w Transformer & - & \textbf{1.914} & 6.72 \\
            \bottomrule
        \end{tabularx}
        \caption{Ablation on ImageNet 256$\times$256 image generation.}
        \label{tab:4}
        \vspace{-1em}
\end{table}

\textbf{Text Generation}\textbf{.} In the image-to-text task image captioning, according to the settings on image generation on the COCO dataset, we tested the caption data of the model based on the model outputting both text and image data using the CIDEr indicator. The results showed that MDM ranked second among all models, as shown in~\cref{tab:2}. While in task VQA, MDM achieves competitive performance, surpassing several traditional models including InstructBLIP, mPLUG-Owl, DREAMLLM, and Emu, although it still trails behind top-performing models in the field as shown in~\cref{tab:2}. In the text-to-text generation task, as shown in~\cref{tab:2}, MDM and the other end-to-end multi-modal models perform worse than well-known traditional models. This discrepancy may be attributed to the fact that these end-to-end models have some deviations in multimodal fusion and learning because they abandon multiple language encoders, visual encoders, and multimodal fusion encoders. However, when compared with the other two end-to-end models, MDM excels, outperforming MonoFormer and surpassing Chameleon on seven out of ten datasets.
\vspace{-0.5em}
\subsection{Discussion}
\vspace{-0.3em}
\subsubsection{Performance Analysis}
\label{sec:EA}
As demonstrated in~\cref{fig:3}, MDM shows the ability to generate image and text simultaneously in multiple rounds of dialogue and perform well in QA\&VQA. Some results even exceed those of GPT-4V, particularly evident in the second and third rows of~\cref{fig:4} which is a hybrid output process for the MDM model. Due to this, we set the model to generate corresponding images for the description text while simultaneously generating image captioning.

This enhanced performance stems from MDM's multi-step selection diffusion decoder, which leverages Mamba's integrated selection and denoising capabilities to maintain focused attention on both textual and visual details. Validating our complexity analysis in Appendix E, MDM demonstrates superior efficiency compared to end-to-end Transformer models when processing long sequences, as shown in~\cref{fig:5}, particularly outperforming other end-to-end multi-modal models for sequences exceeding 1280 tokens.

\vspace{-0.3em}
\subsubsection{Ablations}
\vspace{-0.3em}
\label{sec:ablation}
Our ablation studies examine the impact of both the selection process and Mamba block components. Reducing the number of image/text scan switch sequences from 6 ('\ding{172}\ding{173}\ding{174}\ding{175}/\ding{172}\ding{173}') to 3 ('\ding{172}\ding{173}/\ding{172}'), as shown in~\cref{tab:4}, improves inference speed but degrades image quality, as fewer scan switch sequences limit the model's ability to capture accurate information in complex sequences. Additionally, replacing the Mamba block with the Transformer further deteriorates output image quality, suggesting Mamba's temporal network architecture is better suited for representing diffusion relationships during the denoising process.
\vspace{-0.5em}
\section{Conclusion}
\label{sec:con}
\vspace{-0.5em}
This paper introduces MDM (Multi-Modal Diffusion Mamba), a novel end-to-end architecture that significantly enhances multi-modal processing through two key innovations: a unified diffusion objective and an efficient selection mechanism leveraging Mamba's state-space structure. By integrating variational autoencoder with multi-step selection diffusion, MDM achieves SOTA overall performance in image generation and demonstrates remarkable versatility across various tasks, including image-to-text, text-to-text and text-image-to-text-image. Our comprehensive experiments illustrate that MDM consistently surpasses traditional end-to-end multi-modal models, particularly in processing high-resolution images and long-sequence text, while maintaining computational efficiency. The model's ability to unify different modalities under a single objective, coupled with its superior management of temporal relationships in the diffusion process, establishes a promising direction for future multi-modal architecture.
{
    \small
    \bibliographystyle{ieeenat_fullname}
    \bibliography{main}
}
\clearpage
\setcounter{page}{1}
\maketitlesupplementary

\section{Appendix A}
\label{sec:AppA}
\subsection{Theorem 1}
\label{sec:AppA_th1}
\setcounter{theorem}{0}
\begin{theorem} \label{th_1_app}
    According to Bayes' theorem and the Gaussian distribution density formula, the following calculation relationship of $\frac{p_{data}(y)}{p_{data}(z_{n,t}^{g})}$ is obtained:
    \setcounter{equation}{0}
    \begin{equation}
        \frac{p_{data}(y)}{p_{data}(z_{n,t}^{g})}=\exp{\Bigg( \frac{\|z_{n,t}^{g}\|^{2}}{2}-\frac{\|z_{n,t}^{g}-\sqrt{\bar{\alpha}_{t}^{g}}z_{n,0}^{g}\|^{2}}{2(1-\bar{\alpha}_{t}^{g})}\Bigg)}
        \label{eq:100}
    \end{equation}
\end{theorem}

\begin{proof}
    According to~\cite{loudiscrete}, from Bayes' theorem, we express the posterior probability as:
    \begin{equation}
        p_{data}(z_{n,0}^{g} | z_{n,t}^{g}) = \frac{p(z_{n,t}^{g} | z_{n,0}^{g}) p_{data}(z_{n,0}^{g})}{p(z_{n,t}^{g})}.
    \end{equation}
    Rearranging, we obtain:
    \begin{equation}
        \frac{p_{data}(z_{n,0}^{g})}{p_{data}(z_{n,t}^{g})} = \frac{p(z_{n,t}^{g} | z_{n,0}^{g})}{p(z_{n,t}^{g})}.
    \end{equation}
    Given the real data $z_{n,0}^{g}$, the probability of the diffused noise state is $p(z_{n,t}^{g} | z_{n,0}^{g})$. $p(z_{n,t}^{g})$ is the marginal distribution of all possible $z_{n,0}^{g}$ after diffusion.
    
    The forward noise addition process in the diffusion model is defined as follows:
    \begin{equation}
        z_{n,t}^{g} = \sqrt{\bar{\alpha}_{t}^{g}} z_{n,0}^{g} + \sqrt{1 - \bar{\alpha}_{t}^{g}} \epsilon_{n,t}^{g}, \epsilon_{n,t}^{g} \sim \mathcal{N}(0, I),
    \end{equation}
    and it can be seen that given $z_{n,0}^{g}$, $z_{n,t}^{g}$ obeys the Gaussian distribution:
    \begin{equation}
        p(z_{n,t}^{g} | z_{n,0}^{g}) = \mathcal{N}(z_{n,t}^{g} ; \sqrt{\bar{\alpha}_{t}^{g}} z_{n,0}^{g}, (1 - \bar{\alpha}_{t}^{g}) I),
    \end{equation}
    where this conditional probability indicates that $z_{n,t}^{g}$ is a Gaussian distribution with $\sqrt{\bar{\alpha}_{t}^{g}} z_{n,0}^{g}$ as mean and $(1 - \bar{\alpha}_{t}^{g}) I$ as variance.

    Then, for the marginal distribution $p(z_{n,t}^{g})$ can be calculated by integration:
    \begin{equation}
        p(z_{n,t}^{g}) = \int p(z_{n,t}^{g} | z_{n,0}^{g}) p_{data}(z_{n,0}^{g}) dz_{n,0}^{g}.
    \end{equation}
    Typically, we assume that the underlying distribution of the data follows a standard Gaussian:
    \begin{equation}
        p_{data}(z_{n,0}^{g}) = \mathcal{N}(z_{n,0}^{g} ; 0, I),
    \end{equation}
    since the convolution of two Gaussian distributions is a Gaussian distribution, $p(z_{n,t}^{g})$ is still a Gaussian distribution:
    \begin{equation}
        p(z_{n,t}^{g}) = \mathcal{N}(z_{n,t}^{g} ; 0, I).
    \end{equation}

    Combining the above derivation, we get:
    \begin{equation}
        \frac{p_{data}(z_{n,0}^{g})}{p_{data}(z_{n,t}^{g})} = \frac{p(z_{n,t}^{g} | z_{n,0}^{g})}{p(z_{n,t}^{g})}.
    \end{equation}
    Then, substitute into the Gaussian distribution density formula:
    \begin{equation}
        \frac{p(z_{n,t}^{g} | z_{n,0}^{g})}{p(z_{n,t}^{g})} = \frac{\exp{(-\frac{\|z_{n,t}^{g}-\sqrt{\bar{\alpha}_{t}^{g}}z_{n,0}^{g}\|^{2}}{2(1-\bar{\alpha}_{t}^{g})})}}{\exp{(-\frac{\|z_{n,t}^{g}\|^{2}}{2})}}.
    \end{equation}
    Further sorting, thus, we derive~\cref{eq:100}, completing the proof.
\end{proof}

\subsection{Theorem 2}
\label{sec:AppA_th3}
\begin{theorem} \label{th_3_app}
    Given the denoising process modeled by a score-based probability ratio function $s_{\theta}(z_{n,t}^{g})$, defined as $s_{\theta}=\frac{p_{data}(z_{n,0}^g)}{p_{data}(z_{n,t}^g)}$, this paper defines a learnable approximation using a parameterized score function $f_{\theta}$, such that the probability ratio can be estimated as:
    \begin{equation}
        s_{\theta}(z_{n,t}^{g}) = \frac{\exp{(f_{\theta}(z_{n,t}^{g},z_{n,0}^{g}))}}{\sum_{y\in z_{n,0:t-1}^{g}}\exp{(f_{\theta}(z_{n,t}^{g},y))}},
    \label{eq:300}
    \end{equation}
\end{theorem}

\begin{proof}
    To derive \cref{eq:300}, we start from the definition of the score-based probability ratio:
    \begin{equation}
        s_{\theta}(z_{n,t}^{g}) = \frac{p_{\theta}(z_{n,0}^{g})}{p_{\theta}(z_{n,t}^{g})}.
    \label{eq:301}
    \end{equation}
    
    Using Bayes' theorem, we can express the conditional probability as:
    \begin{equation}
        p_{\theta}(z_{n,0}^{g} | z_{n,t}^{g}) = \frac{p(z_{n,t}^{g} | z_{n,0}^{g}) p_{\theta}(z_{n,0}^{g})}{p(z_{n,t}^{g})}.
    \label{eq:302}
    \end{equation}
    
    Taking the logarithm on both sides, we define a learnable function $f_{\theta}(z_{n,t}^{g},z_{n,0}^{g})$ that approximates:
    \begin{equation}
        f_{\theta}(z_{n,t}^{g},z_{n,0}^{g}) \approx \log p_{\theta}(z_{n,0}^{g} | z_{n,t}^{g}).
    \label{eq:303}
    \end{equation}
    
    Given the forward diffusion process follows:
    \begin{equation}
        p(z_{n,t}^{g} | z_{n,0}^{g}) = \mathcal{N}(z_{n,t}^{g}; \sqrt{\bar{\alpha}_{t}^{g}} z_{n,0}^{g}, (1 - \bar{\alpha}_{t}^{g}) I),
    \label{eq:304}
    \end{equation}
    and the marginal distribution:
    \begin{equation}
        q(z_{n,t}^{g}) \approx \mathcal{N}(z_{n,t}^{g}; 0, I),
    \label{eq:305}
    \end{equation}
    we obtain:
    \begin{equation}
        f_{\theta}(z_{n,t}^{g},z_{n,0}^{g}) = -\frac{\| z_{n,t}^{g} - \sqrt{\bar{\alpha}_{t}^{g}} z_{n,0}^{g} \|^2}{2 (1 - \bar{\alpha}_{t}^{g})} + \frac{\| z_{n,t}^{g} \|^2}{2}.
    \label{eq:306}
    \end{equation}
    
    To ensure numerical stability and gradient optimization, we normalize $s_{\theta}(z_{n,t}^{g})$ using softmax over the set of possible denoising states:
    \begin{equation}
        s_{\theta}(z_{n,t}^{g}) = \frac{\exp{(f_{\theta}(z_{n,t}^{g},z_{n,0}^{g}))}}{\sum_{y\in z_{n,0:t-1}^{g}}\exp{(f_{\theta}(z_{n,t}^{g},y))}}.
    \label{eq:307}
    \end{equation}
    
    Thus, we have derived \cref{eq:300}, which provides a parameterized score function for probability ratio estimation.
\end{proof}

\subsection{Theorem 3}
\label{sec:AppA_th2}
\begin{theorem} \label{th_2_app}
    To achieve the optimal \textbf{s}core \textbf{e}ntropy \cite{loudiscrete} which is demonstrated on~\cref{eq:202}, the selection step choose $j$ items where each $z_{n,t}^g$ satisfies $se = 0$, i.e.,
    \begin{equation}
        s_{\theta}(z_{n,t}^g) \approx \frac{p_{data}(y)}{p_{data}(z_{n,t}^{g})}
    \label{eq:200}
    \end{equation}
\end{theorem}

\begin{proof}
    To prove the~\cref{th_2_app}, we divide this proof into three parts: The first is to determine \textbf{the optimization target of the model approximation}. The second is to determine \textbf{the iterative process of the model optimization target}. The third is to prove \textbf{the convergence validity of the iterative process}.

    1) \textbf{The optimization target of the model approximation}
    
    According to the denoising score entropy proposed by Lou ~\etal ~\cite{loudiscrete}, the Mamba block loss function can be defined as follows:
    \begin{equation}
        L_{se} = \mathbb{E}_{z_{n,0}^{g}\sim p_{0},z_{n}^{g}\sim p(\cdot|z_{n,0}^{g})} se
        \label{eq:201}
    \end{equation}
    
    To minimize the loss function, the $se$ should be closed to value 0. And based on the score entropy loss~\cite{loudiscrete}, the $se$ can be described as:
    \begin{equation}
        \begin{split}
        se = \sum_{y\in z_{n,0:t-1}^{g}} \omega_{z_{n,t}^{g}}^{g} \bigg( & s_{\theta}(z_{n,t}^g) - \frac{p_{data}(y)}{p_{data}(z_{n,t}^{g})} \log{s_{\theta}(z_{n,t}^g)} \\
        & + K\left(\frac{p_{data}(y)}{p_{data}(z_{n,t}^{g})}\right) \bigg),
        \end{split}
        \label{eq:202}
    \end{equation}
    where $K(a)=a(\log{a}-1)$ a normalization term that ensures the loss is non-negative. And weights $\omega_{z_{n,t}^{g}}^{g}\in (0,1)$ can adjust the weights assigned to different noise latent representations. This can improve optimization efficiency by explicitly selecting important point pairs. For example, higher weights can be assigned to noise latent representations that may introduce larger errors within a specific range, thereby guiding the update of the model. And ultimately control the final total $se$ to be close to 0. And $s_{\theta}(z_{n,t}^g)$ is $n$-th noise latent representation of the model predicted score ratio at $t$-th denoising step.
    
    To determine the necessary conditions for minimizing \( se \), we compute the partial derivative with respect to \( s_{\theta}(z_{n,t}^{g}) \):
    \begin{equation}
        \frac{\partial se}{\partial s_{\theta}(z_{n,t}^{g})} = \sum_{y\in z_{n,0:t-1}^{g}} \omega_{z_{n,t}^{g}}^{g} \bigg( 1 - \frac{p_{\text{data}}(y)}{p_{\text{data}}(z_{n,t}^{g})} \frac{1}{s_{\theta}(z_{n,t}^{g})} \bigg).
    \end{equation}

    Setting the gradient to zero for optimization,
    \begin{equation}
        1 - \frac{p_{\text{data}}(y)}{p_{\text{data}}(z_{n,t}^{g})} \frac{1}{s_{\theta}(z_{n,t}^{g})} = 0.
    \end{equation}
    
    Rearranging the terms, we obtain:
    \begin{equation}
        s_{\theta}(z_{n,t}^{g}) = \frac{p_{\text{data}}(y)}{p_{\text{data}}(z_{n,t}^{g})}.
        \label{eq:optimal_condition}
    \end{equation}
    
    Thus, at the optimal solution, the predicted score function must exactly match the empirical probability ratio.

    For model parameters \( \theta \), we analyze the gradient:
    \begin{equation}
        \begin{split}
            \frac{\partial se}{\partial \theta} = \sum_{y\in z_{n,0:t-1}^{g}} \omega_{z_{n,t}^{g}}^{g} \bigg( \frac{\partial s_{\theta}(z_{n,t}^{g})}{\partial \theta} - \\
            \frac{p_{\text{data}}(y)}{p_{\text{data}}(z_{n,t}^{g})} \frac{1}{s_{\theta}(z_{n,t}^{g})} \frac{\partial s_{\theta}(z_{n,t}^{g})}{\partial \theta} \bigg).
        \end{split}
    \end{equation}
    
    For gradient convergence, we set the derivative to zero:
    \begin{equation}
        \frac{\partial s_{\theta}(z_{n,t}^{g})}{\partial \theta} \bigg( 1 - \frac{p_{\text{data}}(y)}{p_{\text{data}}(z_{n,t}^{g})} \frac{1}{s_{\theta}(z_{n,t}^{g})} \bigg) = 0.
        \label{eq:delta}
    \end{equation}
    
    Since the gradient term \( \frac{\partial s_{\theta}(z_{n,t}^{g})}{\partial \theta} \) is nonzero for model updates, the following condition must hold:
    \begin{equation}
        1 - \frac{p_{\text{data}}(y)}{p_{\text{data}}(z_{n,t}^{g})} \frac{1}{s_{\theta}(z_{n,t}^{g})} = 0,
    \end{equation}
    which again yields the optimal condition:
    \begin{equation}
        s_{\theta}(z_{n,t}^{g}) = \frac{p_{\text{data}}(y)}{p_{\text{data}}(z_{n,t}^{g})}.
    \end{equation}

    In summary, the necessary conditions for minimizing the Score Entropy Loss and ensuring the optimal score function are:
    \begin{itemize}
        \item The predicted score function must satisfy:
        \begin{equation}
            s_{\theta}(z_{n,t}^{g}) = \frac{p_{\text{data}}(y)}{p_{\text{data}}(z_{n,t}^{g})}.
        \end{equation}
        \item The gradient with respect to the model parameters must satisfy:
        \begin{equation}
            \frac{\partial s_{\theta}(z_{n,t}^{g})}{\partial \theta} \bigg( 1 - \frac{p_{\text{data}}(y)}{p_{\text{data}}(z_{n,t}^{g})} \frac{1}{s_{\theta}(z_{n,t}^{g})} \bigg) = 0.
        \end{equation}
    \end{itemize}
    
    These conditions imply that when the model learns the correct probability ratio, the gradient becomes zero, leading to optimal convergence of the Score Entropy Loss. Therefore, optimizing \( s_{\theta}(z_{n,t}^{g}) \) to match \( \frac{p_{\text{data}}(y)}{p_{\text{data}}(z_{n,t}^{g})} \) is both a necessary and sufficient condition for achieving the lowest possible loss.

    Based on~\cref{eq:delta}, $\theta = \lbrace H_{n,t}^{g}, A, B, C, D, \Delta \rbrace$ represent the state space in the block. We can obtain the selected noise latent representation $z_{n,t}^g$ by updating the computation in the state space architecture from Mamba-2~\cite{gu2023mamba}, which can be defined as follows:
    \begin{equation}
        H_{n,t}^{g} = \bar{A}H_{n,t-1}^{g} + \bar{B}z_{n,t}^{g}
        \label{eq:210}
    \end{equation}
    \begin{equation}
        z_{n-1,t}^{g} = CH_{n,t}^{g} + Dz_{n,t}^{g}
        \label{eq:211}
    \end{equation}
    \begin{equation}
        \bar{A} = \exp{(\Delta A)}
        \label{eq:212}
    \end{equation}
    \begin{equation}
        \bar{B} = (\Delta A)^{-1}\cdot (\exp{(\Delta A)}-I)\cdot \Delta B
        \label{eq:213}
    \end{equation}
    where $H_{n,t}^{g}$ represents the hidden state representation, $A$ and $B$ control the evolution of hidden states and latent space noise vector inputs, respectively, $C$ governs the hidden state representation of the target output and $D$ manages the nonlinear skip connection for latent space noise vector inputs. $\Delta$ denotes the learnable time parameter.
    
    2) \textbf{The iterative process of the model optimization target}    
    Considering the parameters in $\theta$, they are updated by the following steps.
    First, \textbf{the update of $A$ and $\bar{A}$}. Given that $\bar{A}$ controls the recursive evolution of hidden state $H_{n,t}^{g}$ based on $A$ and $\Delta$, we can gain the relationship in~\cref{eq:212}. So, the gradient can be described as follows:
    \begin{equation}
        \frac{\partial \mathcal{L}}{\partial A} = \frac{\partial \mathcal{L}}{\partial \bar{A}} \cdot \frac{\partial \bar{A}}{\partial A}
        \label{eq:214}
    \end{equation}
    where
    \begin{equation}
        \frac{\partial \bar{A}}{\partial A} = \Delta \cdot \exp{(\Delta A)}
        \label{eq:215}
    \end{equation}
    then through backpropagation to calculate the gradient of $\mathcal{L}$ to $\bar{A}$ and combined with the chain rule to update $A$.

    Second, \textbf{the update of $B$ and $\bar{B}$}. Given that the definition of $\bar{B}$ in~\cref{eq:213}, the gradient can be described as follows (familiar with the update rule of $A$):
    \begin{equation}
        \frac{\partial \mathcal{L}}{\partial B} = \frac{\partial \mathcal{L}}{\partial \bar{B}} \cdot \frac{\partial \bar{B}}{\partial B}
        \label{eq:216}
    \end{equation}
    where gradient transfer involves matrix derivation, which requires considering the derivative rule of matrix multiplication. Finally, the chain rule depends on the gradients of $\Delta A$ and $\Delta B$.

    Third, \textbf{the update of $C$}. Given that $C$ controls the hidden state and its direct contribution to the output $z_{n-1,t}^g$ is as~\cref{eq:211} defined, the gradient can be described as follows:
    \begin{equation}
        \frac{\partial \mathcal{L}}{\partial C} = \frac{\partial \mathcal{L}}{\partial z_{n-1,t}^g} \cdot \frac{\partial z_{n-1,t}^g}{\partial C}
        \label{eq:217}
    \end{equation}
    where
    \begin{equation}
        \frac{\partial z_{n-1,t}^g}{\partial C} = H_{n,t}^g
        \label{eq:218}
    \end{equation}

    So the update rule can be described as follows:
    \begin{equation}
        C\leftarrow C - \eta\frac{\partial \mathcal{L}}{\partial C}
        \label{eq:219}
    \end{equation}
    where $\eta$ is the learning rate.

    Fourth, \textbf{the update of $D$}. Given that $D$ governs the skip connection and directly act on $z_{n,t}^g$, the gradient can be defined as follows:
    \begin{equation}
        \frac{\partial \mathcal{L}}{\partial D} = \frac{\partial \mathcal{L}}{\partial z_{n-1,t}^g} \cdot \frac{\partial z_{n-1,t}^g}{\partial D}
        \label{eq:220}
    \end{equation}
    where
    \begin{equation}
        \frac{\partial z_{n-1,t}^g}{\partial D} = z_{n,t}^g
        \label{eq:221}
    \end{equation}

    Fifth, \textbf{the update of $\Delta$}. $\Delta$ denotes the learnable time parameter and affects the dynamic behavior of $\bar{A}$ and $\bar{B}$. So the gradient can be defined as follows:
    \begin{equation}
        \frac{\partial \mathcal{L}}{\partial \Delta} = \frac{\partial \mathcal{L}}{\partial \bar{A}}\cdot \frac{\partial \bar{A}}{\partial \Delta} + \frac{\partial \mathcal{L}}{\partial \bar{B}}\cdot \frac{\partial \bar{B}}{\partial \Delta}
        \label{eq:222}
    \end{equation}
    where
    \begin{equation}
        \frac{\partial \bar{A}}{\partial \Delta}=A\cdot \exp{(\Delta A)}
        \label{eq:223}
    \end{equation}
    \begin{align}
        f(A, B, \Delta) = & -(\Delta A)^{-1} A (\Delta A)^{-1} (\exp(\Delta A) - I) \Delta B \notag \\
        & + (\Delta A)^{-1} (A \exp(\Delta A)) \Delta B \notag \\
        & + (\Delta A)^{-1} (\exp(\Delta A) - I) B
        \label{eq:224}
    \end{align}

    In this problem, the structure of the state space model and the diffusion model provide theoretical support for the strong convexity of the loss function and the Lipschitz property of the gradient. First, the stability of the state space model leads to the hidden state update equation:
    \begin{equation}
        H_{n,t}^g = \bar{A} H_{n,t-1}^g + \bar{B} z_{n,t}^g
        \label{eq:229}
    \end{equation}
    where $\bar{A} = \exp(\Delta A), \bar{B} = (\Delta A)^{-1} (\exp(\Delta A) - I) \Delta B$ is generated via matrix exponential. It has the following characteristics:
    \begin{itemize}
        \item If $A$ is a stable matrix (all eigenvalues have negative real parts), then the modulus of the eigenvalues of $\bar{A}$ is less than $1$, which ensures that the hidden state does not diverge.
        \item The state update equation is linear, so the gradient of the parameters $A,B,C,D$ is linearly solvable, making it easy to optimize.
    \end{itemize}

    Secondly, given the characteristics of the diffusion model, there is a score ratio prediction loss function:
    \begin{equation}
        \mathcal{L} = \mathbb{E}_{z_{n,t}^g, p} \left[ \| \frac{p_{data}(y)}{p_{data}(z_{n,t}^{g})} - s_\theta(z_{n,t}^g) \|_2^2 \right]
        \label{eq:230}
    \end{equation}
    where $\mathcal{L}$ is in squared error form and is therefore a convex function (subconvexity). Then the gradient can be expressed as follows:
    \begin{equation}
        \nabla_\theta \mathcal{L} = 2 \mathbb{E}_{z_{n,t}^g, p} \left[ \| (\frac{p_{data}(y)}{p_{data}(z_{n,t}^{g})} - s_\theta(z_{n,t}^g)) \nabla_\theta s_\theta(z_{n,t}^g) \|_2^2 \right]
        \label{eq:231}
    \end{equation}
    where the gradient is a linear combination of $\theta$ and satisfies the Lipschitz continuity condition.

    To sum up, combined with the model parameters $\theta={A,B,C,D,\Delta}$, there is the following convergence of the specific parameter updating process.

    First for the hidden state update:
    \begin{equation}
        H_{n,t}^g = \bar{A} H_{n,t-1}^g + \bar{B} z_{n,t}^g
        \label{eq:232}
    \end{equation}
    where $A$ is a stable matrix, $\bar{A}$ is stable, ensuring that the hidden state does not diverge.

    Second for output calculation:
    \begin{equation}
        z_{n-1,t}^g = C H_{n,t}^g + D z_{n,t}^g
        \label{eq:233}
    \end{equation}
    and it is a linear transformation, which ensures the stability of the gradient solution for $C$ and $D$.

    Third for time step parameters $\Delta$, it is a learnable parameter of the time scale, which is directly related to the discretization in the state space model. It is updated by the chain rule as follows:
    \begin{equation}
        \frac{\partial \mathcal{L}}{\partial \Delta} = \frac{\partial \mathcal{L}}{\partial \bar{A}} \cdot \frac{\partial \bar{A}}{\partial \Delta} + \frac{\partial \mathcal{L}}{\partial \bar{B}} \cdot \frac{\partial \bar{B}}{\partial \Delta}
        \label{eq:234}
    \end{equation}
    among this, in the discretization formula, $\bar{A}$ and $\bar{B}$ are exponential functions with continuous and differentiable gradients, which are easy to converge.

    3) \textbf{The convergence validity of the iterative process}

    In order to ensure the convergence of the above iterative process, the following conditions usually need to be met:
    \begin{itemize}
        \item The convergent objective function $\mathcal{L}$ is a continuously differentiable function with respect to parameter $\theta$ and it is strongly convex or subconvex (at least a convex function).
        \item Make sure the learning rate satisfies $0<\eta<2/L$ where $L$ is the Lipschitz constant for the gradient $\nabla_\theta\mathcal{L}$ of the convergent objective function (the upper bound on the rate of change of the gradient).
        \item The matrix $\bar{A}$ (generated by discretization) is stable, that is, the magnitude of its eigenvalues is less than 1.
    \end{itemize}

    When the above convergence conditions are met, assuming that the convergence target $\mathcal{L}$ function is a $\mu$-strongly convex function (strong convexity is a stricter form of convex function), the convergence of gradient descent can be proved by the following formula. First, the updated formula for gradient descent is given:
    \begin{equation}
        \theta^{k+1} = \theta^k - \eta \nabla_{\theta} \mathcal{L}(\theta^k)
        \label{eq:226}
    \end{equation}
    where $\theta^k$ is the parameter vector at the k-th iteration.

    Secondly, the properties of strongly convex functions are given, that is, if the convergent objective function $\mathcal{L}$ is $\mu$-strongly convex and the Lipschitz constant of the gradient is $L$, then the error of the gradient descent method will converge at an exponential rate:
    \begin{equation}
        \mathcal{L}(\theta^k) - \mathcal{L}(\theta^*) \leq \rho^k \big(\mathcal{L}(\theta^0) - \mathcal{L}(\theta^*)\big)
        \label{eq:227}
    \end{equation}
    where $\rho = 1 - 2\eta\mu$ is the convergence rate ($0 < \rho < 1$), and $\theta^*$ is the global optimum.

    Third, if the Lipschitz gradient condition is satisfied, that is, $\nabla_\theta \mathcal{L}$ is $L$-Lipschitz continuous:
    \begin{equation}
        \|\nabla_\theta \mathcal{L}(\theta_1) - \nabla_\theta \mathcal{L}(\theta_2)\| \leq L \|\theta_1 - \theta_2\|
        \label{eq:228}
    \end{equation}
    then selecting a learning rate $0 < \eta < \frac{2}{L}$ ensures convergence.
    \begin{algorithm}\label{algo:1}
    \caption{Gradient Descent Algorithm}
    \begin{algorithmic}[h]
        \State \textbf{Input:} Initialize parameters $A$, $B$, $C$, $D$, and $\Delta$.
        \Repeat
            \State Calculate the loss $\mathcal{L}$.
            \State Compute the gradient of $\mathcal{L}$ with respect to $A$, $B$, $C$, $D$, and $\Delta$ using the chain rule.
            \State Update each parameter using the gradient descent rule.
            \State Perform backpropagation to compute: \\
            $\nabla_{\theta}\|\frac{p_{data}(y)}{p_{data}(z_{n,t}^{g})} - s_\theta(z_{n,t}^{g})\|_2^2$.
        \Until{convergence}
        \end{algorithmic}
    \end{algorithm}
    
    In general, the process of update and convergence can be summarized in Algorithm 1. Through repeated iterations, the model parameter $\theta$ will be gradually optimized, so that the convergence objective function $\mathcal{L}$ will be converged and $se$ gradually approaches 0, that is, $s_\theta$ approaches $\frac{p_{data}(y)}{p_{data}(z_{n,t}^{g})}$. Then $j$ items of noise latent representation $z_{n,t}^g$ that satisfy all the above conditions will be selected, and the model will proceed to the next step of denoising in the direction of these $j$ items.

    Above all, in the inference stage, the model will choose the best noise latent representation of image patch or text embedding, including $j$ items to restore the image or text. Due to this, the model has already learned from the datasets that should be focused on and ignored. Compared with the Transformer models, which need to calculate all image patches or text embeddings, it will shorten the inference time when generating high-resolution images or long-sequence text. The results are shown in the main paper Section 5.3.1 Performance Analysis.
    
\end{proof}

\section{Appendix B}
\label{sec:AppB}
\subsection{Denoising process based on DPM-Solver}
Based on the diffusion denoising model trained by Score Entropy Loss, we hope to combine DPM-Solver (Diffusion Probabilistic Model Solver)\cite{lu2022dpm} in the inference stage to reduce sampling steps and improve inference efficiency.

DPM-Solver is a high-order ODE-solving method for diffusion models. It constructs partial differential equations (ODEs) and uses numerical solution techniques to accelerate the diffusion denoising process. It can restore high-quality data from Gaussian noise in a minimal number of steps (such as 10 steps) without sacrificing model performance.

The core idea of DPM-Solver is to reformulate the inverse diffusion process of the diffusion model as an ordinary differential equation (ODE) and solve it efficiently using numerical methods. For the standard diffusion model, we have:
\begin{equation}
    \frac{dz_{n,t}^g}{dt} = -\frac{1}{2}\beta_t z_{n,t}^g + \sqrt{\beta_t}\epsilon_{n,t}^g, \quad \epsilon_{n,t}^g \sim \mathcal{N}(0, I).
\end{equation}
DPM-Solver estimates $\epsilon_{\theta}(z_{n,t}^{g},t)$ by denoising the score matching, which can be rewritten as:
\begin{equation}
    \frac{dz_{n,t}^g}{dt} = f_{\theta}(z_{n,t}^{g},t),
\end{equation}
where the formula describes the rate of change of the latent variable $z_{n,t}^{g}$ in the time $t$ dimension, and its evolution process can be accelerated by numerical solution methods.

In the Mamba decoder trained with Score Entropy Loss, we learn:
\begin{equation}
    s_{\theta}(z_{n,t}^{g}) = \frac{\exp{(f_{\theta}(z_{n,t}^{g},z_{n,0}^{g}))}}{\sum_{y\in z_{n,0:t-1}^{g}}\exp{(f_{\theta}(z_{n,t}^{g},y))}}.
\end{equation}

Therefore, in the DPM-Solver framework, we hope to use this ratio's gradient information to directly construct the ODE and reduce the number of sampling steps during inference.

First, we need to compute denoised ODE. DPM-Solver uses Score Matching technology~\cite{lu2022dpm} to predict the noise $\epsilon_{\theta}(z_{n,t}^{g},t)$ through a neural network, and then calculates it according to the denoising ODE:
\begin{equation}
    \frac{dz_{n,t}^g}{dt} = -\frac{1}{2}\beta_t \left( \frac{z_{n,t}^g - \sqrt{\bar{\alpha}_t^g}z_{n,0}^g}{1 - \bar{\alpha}_t^g} \right),
\end{equation}
furthermore, we can calculate based on Score Entropy~\cite{loudiscrete}:
\begin{equation}
    \frac{dz_{n,t}^g}{dt} = -\frac{1}{2}\beta_t s_\theta(z_{n,t}^g)\nabla_z \log p_\theta(z_{n,0}^g | z_{n,t}^g),
\end{equation}
where $\nabla_z \log p_\theta(z_{n,0}^g | z_{n,t}^g)$ is calculated by $se$, $s_\theta(z_{n,t}^g$ is predicted probability ratios through neural networks. This formula describes the ODE trajectory from the noisy state $z_{n,t}^g$ to the denoised state $z_{n,0}^g$.

We then use DPM-Solver to perform inference. For the first-order approximation method, the basic form of DPM-Solver is the first-order ODE approximation:
\begin{equation}
    z_{n,t}^g \approx z_{n,t-\Delta t}^g - \frac{1}{2} \beta_t \left( \frac{z_{n,t}^g - \sqrt{\bar{\alpha}_t^g} z_{n,0}^g}{1 - \bar{\alpha}_t^g} \right) \Delta t,
\end{equation}
by using $s_{\theta}(z_{n,t}^g)$ calculated by Score Entropy Loss, we can further rewrite the formula:
\begin{equation}
    z_{n,t}^g \approx z_{n,t-\Delta t}^g - \frac{1}{2} \beta_t s_\theta(z_{n,t}^g) \nabla_z \log p_\theta(z_{n,0}^g | z_{n,t}^g) \Delta t.
\end{equation}
The formula can be directly used to update the denoising process to achieve efficient sampling iteratively.

Furthermore, DPM-Solver uses second-order numerical methods~\cite{lu2022dpm} to improve accuracy:
\begin{equation}
    z^g_{n,t} = z^g_{n,t-\Delta t} + \frac{\Delta t}{2} \left[ f_\theta(z^g_{n,t}, t) + f_\theta(z^g_{n,t-\Delta t}, t - \Delta t) \right]
\end{equation}
which allows us to complete denoising inference in a very small number of iterations (e.g., 10-20 steps), significantly speeding up the computation compared to normal diffusion sampling (e.g., 1000 steps).
\begin{algorithm}\label{algo:mamba_dpm}
    \caption{Mamba-Based Inference with DPM-Solver}
    \begin{algorithmic}[h]
        \State \textbf{Input:} Noisy latent state \( z_{n,t}^{g} \).
        \Repeat
            \State Predict the score function \( s_{\theta}(z_{n,t}^{g}) \) for computing the denoising ODE.
            \State Apply DPM-Solver update rule:
            $z_{n,t}^{g} \leftarrow z_{n,t-\Delta t}^{g} + \frac{\Delta t}{2} \left[ f_{\theta}(z_{n,t}^{g}, t) + f_{\theta}(z_{n,t-\Delta t}^{g}, t - \Delta t) \right]$.
        \Until{gain the $z_{n,t}^{g}$}
    \end{algorithmic}
\end{algorithm}

\section{Appendix C}
\label{sec:AppC}
\subsection{Model Configuration}
\vspace{-1.2em}
\label{sec:AppC_MC}
\setcounter{table}{0}
\begin{table}[H]
    \centering
        \renewcommand{\arraystretch}{1}  
        \scriptsize
        \begin{tabular}{ll}
            \toprule
            Configuration & Value\\
            \midrule
            Size & 7B \\
            Mamba block & 49 \\
            Hidden Dimension & 2048 \\
            GFlops & 424 \\
            Optimizer & AdamW \\
            Learning Rate & 0.0001 \\
            Weight Decay & - \\
            Training Epochs & 1 \\
            Sampling step & 500000 \\
            EMA & 0.9999 \\
            Patch size & 2$\times$2 \\
            Maximum Token Length & 512 \\
            \bottomrule
        \end{tabular}
        \vspace{-1em}
        \caption{Parameter settings for MDM.}
        \label{tab:1_app}
\end{table}
\vspace{-1.2em}
\section{Appendix D}
\label{sec:AppD}
\subsection{SentencePiece (Unigram BPE)}
SentencePiece (Unigram BPE) \cite{kudo2018sentencepiece} provides an optimal subword-based tokenization approach that enables improved generalization and adaptability for handling both textual and multimodal data.

\subsubsection{Theoretical Background}
SentencePiece employs a probabilistic model based on a Unigram Language Model (ULM), where each sentence \( x \) is decomposed into a sequence of subwords \( s_i \) with a likelihood function:
\begin{equation}
    p(x) = \prod_{i} p(s_i),
\end{equation}
where each subword unit \( s_i \) is assigned a probability estimated from training data. Unlike traditional Byte-Pair Encoding (BPE), which deterministically merges frequent subword pairs, the Unigram BPE method probabilistically learns an optimal vocabulary while gradually discarding subwords with lower contributions.

To train SentencePiece, an initial vocabulary is constructed using all possible subword combinations, after which an iterative Expectation-Maximization (EM) optimization is performed. At each iteration, subwords contributing the least to sequence likelihoods are removed, leading to an optimal vocabulary.

\subsubsection{Training Procedure}

The training of the SentencePiece model is conducted on a large-scale dataset containing both pure-text corpora and multimodal text-image descriptions. Given the multimodal nature of our dataset, we mix textual data from Ultrachat and text descriptions from JourneyDB and ImageNet to ensure cross-modal adaptability.

\paragraph{Dataset Preprocessing:} To prepare the dataset, raw text is extracted, normalized, and formatted as a line-separated corpus file. The dataset mixing strategy follows:
\begin{itemize}
    \item Extract textual information from Ultrachat.
    \item Concatenate textual descriptions from JourneyDB and ImageNet.
    \item Remove redundant, low-quality, or excessively short text samples.
    \item Shuffle the corpus to prevent dataset bias.
\end{itemize}

\paragraph{SentencePiece Model Training:} The SentencePiece Unigram BPE model is trained using the following configuration:
\begin{verbatim}
import sentencepiece as spm
spm.SentencePieceTrainer.train(
    input="text_data.txt",  
    # Training corpus
    model_prefix="unigram_bpe",  
    # Output model prefix
    vocab_size=32000,  
    # Vocabulary size
    model_type="unigram",  
    # Unigram-based BPE
    character_coverage=0.9995,  
    # Coverage for rare characters
    num_threads=8,  
    # Parallel training
    input_sentence_size=1000000,  
    # Sample size
    shuffle_input_sentence=True  
    # Shuffle corpus
)
\end{verbatim}
This results in two key output files: \texttt{unigram\_bpe.model} (binary model for tokenization) and \texttt{unigram\_bpe.vocab} (vocabulary list with probabilities).

\subsubsection{Evaluation and Optimization Strategies}

The effectiveness of the trained tokenization model is evaluated based on tokenization efficiency and generalization capability. The following criteria are considered:
\begin{itemize}
    \item \textbf{Subword Granularity}: The trade-off between word and character-level tokenization.
    \item \textbf{Out-of-Vocabulary (OOV) Rate}: The ability to handle unseen words.
    \item \textbf{Multimodal Alignment}: The compatibility of subword embeddings with image features in the latent space.
\end{itemize}

Given the computational constraints of multimodal diffusion models, we optimize the SentencePiece model with:
\begin{itemize}
    \item Selecting an optimal \texttt{vocab\_size} (\( 16K \)-\( 32K \)) to balance representation and sequence length.
    \item Applying dataset mixture strategies to enhance generalization across different data distributions.
    \item Ensuring tokenization stability by enforcing \texttt{character\_coverage} \( 0.9995 \) to capture rare textual variations.
\end{itemize}

\section{Appendix E}
\label{sec:AppE}
\subsection{Complexity}
Since the size of the noisy latent encoder (VAE) is significantly smaller than that of the diffusion decoder (Mamba), we will focus our analysis on the computational complexity of the diffusion decoder. According to~\cite{qu2024survey}, the complexity of each Mamba block is $\mathcal{O}(LN^2)$, where $L$ is the length of the input data and $N$ refers to the size of each parameter ($\lbrace H_{n,t}^{g}, A, B, C, D, \Delta \rbrace$) in the state space. The diffusion decoder is composed of $M$ Mamba blocks, resulting in an overall computational complexity of $\mathcal{O}(MLN^2)$.

For comparison, consider an equivalent end-to-end transformer model optimized with GQA \cite{zhao2024monoformer, team2024chameleon, ainslie2023gqa}. This model maintains the same input length $L$ and GQA module dimension $N$. With $M$ layers and a grouping parameter $G$, its computational complexity is $\mathcal{O}(ML^{2}N/G)$.

Determining which complexity is superior between $\mathcal{O}(MLN^2)$ and $\mathcal{O}(ML^{2}N/G)$ can be challenging. However, it is important to note that $N$ can be significantly smaller than $L/G$ when $L$ is very large. As a result, the proposed MDM can achieve greater computational efficiency than end-to-end transformer models when processing high-resolution images and long-sequence texts.

\section{Appendix F}
\label{sec:AppF}
\subsection{Image generation}
\label{sec:AppF_UIGR}
\setcounter{figure}{0}
\begin{figure}[H]
  \centering
   \includegraphics[width=1\linewidth]{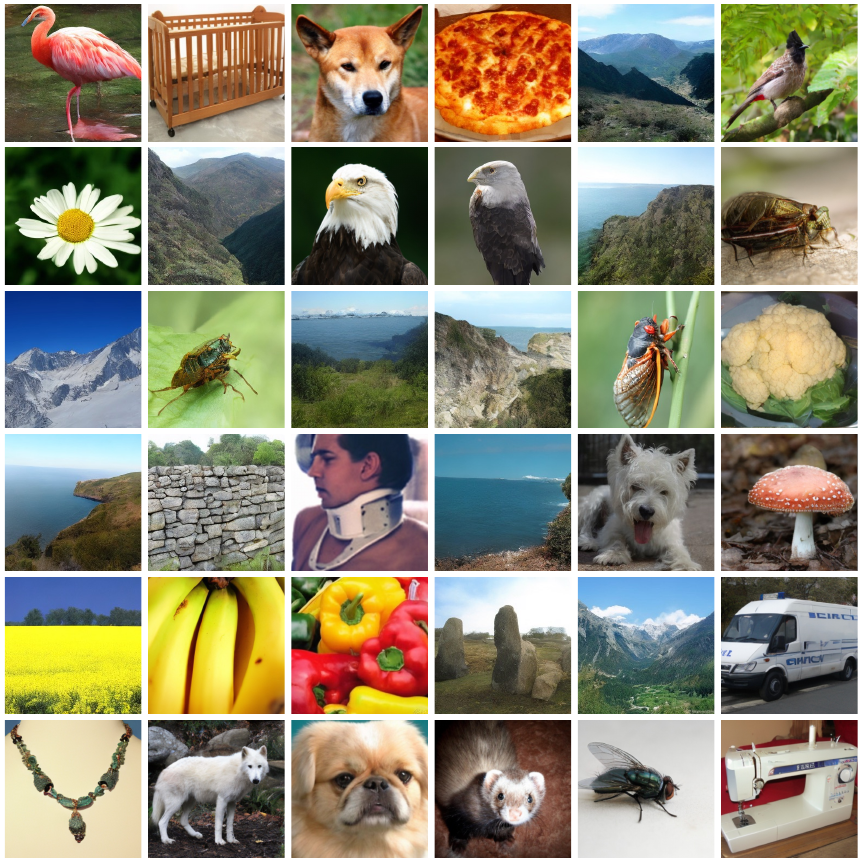}
   \caption{Image generation with CFG on ImageNet~\cite{deng2009imagenet} 256 $\times$ 256.}
   \label{fig:1_app}
\end{figure}

\subsection{Image generation on COCO and Flickr}
\label{sec:AppF_IC}
\begin{figure}[H]
  \centering
   \includegraphics[width=1\linewidth]{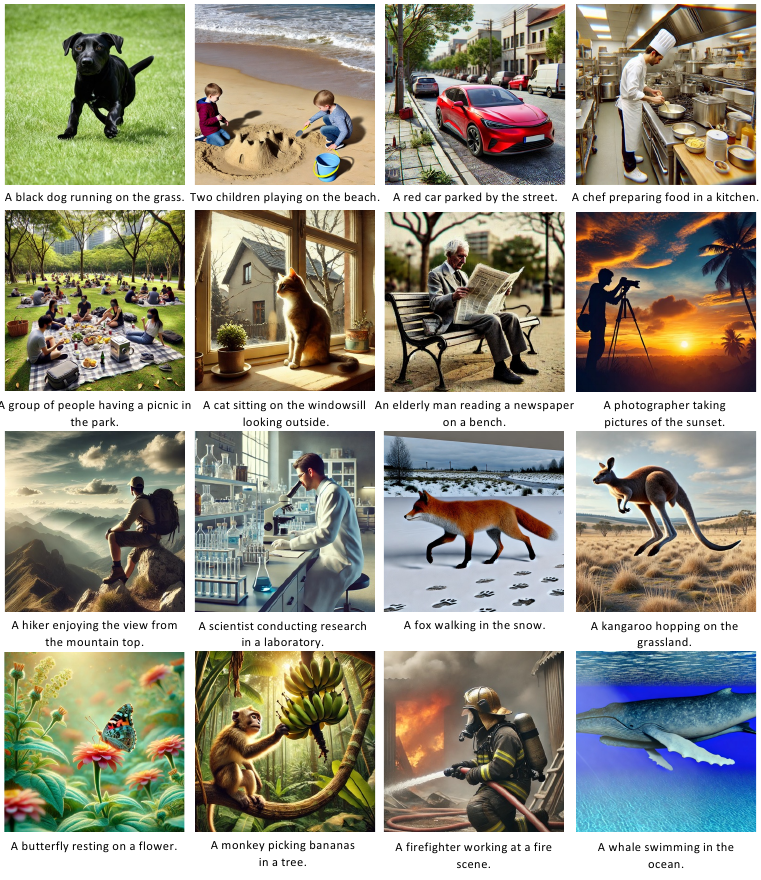}
   \caption{Image generation on COCO~\cite{karpathy2015deep} caption text.}
   \label{fig:2_app}
\end{figure}
\begin{figure}[H]
  \centering
   \includegraphics[width=1\linewidth]{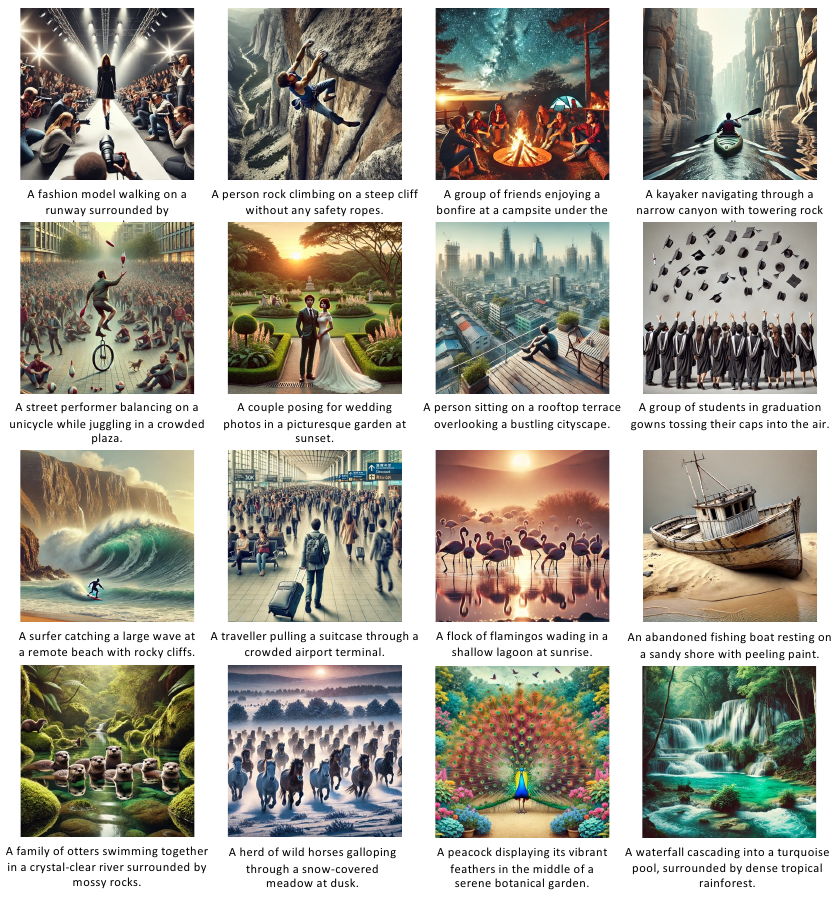}
   \caption{Image generation on Flickr 30K~\cite{young2014image} caption text.}
   \label{fig:3_app}
\end{figure}


\section{Appendix G}
\label{sec:AppG}
\vspace{-1em}
\begin{figure}[H]
  \centering
   \includegraphics[width=1\linewidth]{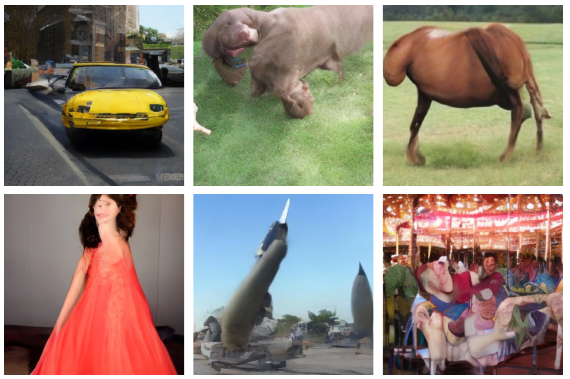}
   \caption{Drawbacks in image generation.}
   \label{fig:5_app}
\end{figure}
\subsection{Drawbacks}
\label{sec:AppG_D}
While MDM demonstrates strong performance across various tasks and enhanced processing speed for high-resolution images and long text sequences (as shown in the main paper Section 5.3.1 Performance Analysis), it faces several limitations. The model shows reduced efficiency when handling low-resolution images or short text sequences, and its overall performance still trails behind traditional multi-modal pre-trained models. Furthermore, the model exhibits hallucination issues. These limitations represent key areas for future improvement.

It can be observed from ~\cref{fig:5_app} that MDM still generates a small number of defective images, such as image deformation, collapse, distortion, and blurring. This may be due to the model's scale being insufficient and limitations in how each modality's data is represented in the decoder. Additionally, the diffusion reduction process might experience some instability, which could lead to subpar sampling results. Therefore, there is still potential for further improvements to the model to address these issues.

The partial performance results of the model on the Flickr 30K dataset reveal significant challenges, particularly when dealing with complex text data that requires generating intricate images, especially those involving people and animals. The model often loses important details, such as facial features and the depiction of limbs. Additionally, it exhibits a tendency to be inefficient and make errors, such as repetitively copying and pasting certain objects, resulting in a dilution of detail for those entities and the generation of instances that do not accurately match the accompanying descriptive language (as shown in~\cref{fig:3_app,fig:6_app}). The main reason for the above problems is that the Flickr 30K dataset emphasizes the correlation between different modal semantics rather than focusing solely on classification or recognition tasks like the COCO dataset. This means that the model needs stronger capabilities for multi-modal semantic understanding. The MDM model employs a unified modal fusion decoder under a constrained scale, which may limit its ability to enhance semantic understanding compared to traditional models. Therefore, the MDM model needs continuous optimization.
\begin{figure}[H]
  \centering
   \includegraphics[width=1\linewidth]{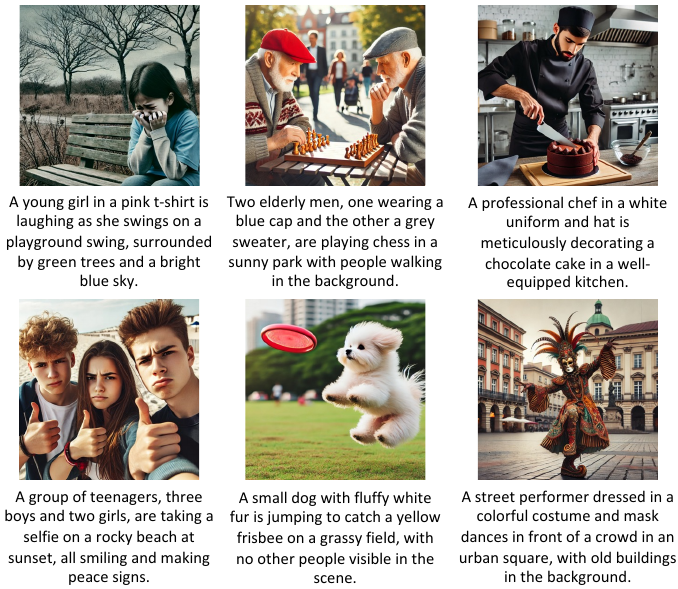}
   \caption{Drawbacks in generating complex captions images.}
   \label{fig:6_app}
\end{figure}

\end{document}